\documentclass[11pt]{article}

\usepackage[preprint]{acl}

\usepackage{times}
\usepackage{latexsym}

\usepackage[T1]{fontenc}

\usepackage[utf8]{inputenc}

\usepackage{microtype}

\usepackage{inconsolata}

\usepackage{graphicx}

\usepackage{amsmath}
\usepackage{amsthm}
\usepackage{graphicx}
\usepackage{multirow}
\usepackage{wrapfig} 
\usepackage{subcaption}
\usepackage{float}

\usepackage{hyperref}
\usepackage{url}
\usepackage{graphicx}
\usepackage{amssymb}
\usepackage{algorithm}
\usepackage{algpseudocode}

\newtheorem{theorem}{Theorem}
\newtheorem{lemma}{Lemma}
\newtheorem{definition}{Definition}
\newtheorem{proposition}{Proposition}
\newtheorem{assumption}{Assumption}

\usepackage{booktabs}
\usepackage[table]{xcolor}
\usepackage{tabularx}
\usepackage{wrapfig}
\usepackage{colortbl,xcolor}
\newcommand{\NA}{\cellcolor{black!8}\textsc{N/A}}
\newcommand{\NR}{\textcolor{black!40}{--}}

\title{RELOOP: \underline{Re}cursive Retrieva\underline{l} with Multi-H\underline{o}p Reas\underline{o}ner and \underline{P}lanners for Heterogeneous QA}

\author{
 \textbf{Ruiyi Yang\textsuperscript{1}},
 \textbf{Hao Xue\textsuperscript{1,2}},
 \textbf{Imran Razzak\textsuperscript{1,3}},
\\
\textbf{Hakim Hacid\textsuperscript{4}},
 \textbf{Flora D. Salim\textsuperscript{1}},
\\
 \textsuperscript{1}University of New South Wales,
 \textsuperscript{2}The Hong Kong University of Science and Technology
(Guangzhou),
 \\
 \textsuperscript{3}Mohamed Bin Zayed University of Artificial Intelligence,
 \textsuperscript{4}Technology Innovation Institute,
\\
 \small{
   \textbf{Correspondence:} \href{mailto:ruiyi.yang@unsw.edu.au}{ruiyi.yang@unsw.edu.au}
 }
}

\begin{document}
\maketitle

\begin{abstract}
Retrieval-augmented generation (RAG) remains brittle on multi-step questions and heterogeneous evidence sources, trading accuracy against latency and token/tool budgets. This paper introduces \textbf{RELOOP}, a structure aware framework using \textbf{Hierarchical Sequence (HSEQ)} that (i) linearize documents, tables, and knowledge graphs into a reversible hierarchical sequence with lightweight structural tags, and (ii) perform structure-aware iteration to collect just-enough evidence before answer synthesis. A Head Agent provides guidance that leads retrieval, while an Iteration Agent selects and expands HSeq via structure-respecting actions (e.g., parent/child hops, table row/column neighbors, KG relations); Finally the head agent composes canonicalized evidence to genearte the final answer, with an optional refinement loop to resolve detected contradictions. Experiments on HotpotQA (text), HybridQA/TAT-QA (table+text), and MetaQA (KG) show consistent EM/F1 gains over strong single-pass, multi-hop, and agentic RAG baselines with high efficiency. Besides, RELOOP exhibits three key advantages: (1) a \textbf{format-agnostic unification} that enables a single policy to operate across text, tables, and KGs without per-dataset specialization; (2) \textbf{guided, budget-aware iteration} that reduces unnecessary hops, tool calls, and tokens while preserving accuracy; and (3) \textbf{evidence canonicalization for reliable QA}, improving answers consistency and auditability. Codes are available at https://anonymous.4open.science/r/RELOOP

\end{abstract}

\section{Introduction}

Large language models (LLMs), such as ChatGPT~\citep{achiam2023gpt}, LLaMA~\citep{dubey2024llama}, Falcon~\citep{zuo2025falcon}, have been increasingly relying on retrieval-augmented generation (RAG) to ground answers in external evidence. With reliable supplementary knowledge offered factual errors are reduced, especially in domain-specific questions, leading to higher accuracy and fewer hallucinations~\citep{zhu2021retrieving, gao2023retrieval, zhao2024retrieval}. Yet state-of-the-art pipelines, remain brittle on multi-step questions and heterogeneous sources, and still struggle to cope with the following challenges:

$\mathbf{C_1}:\textbf{Coverage in Single-pass Retrievers}$:
Single-pass pipelines (retrieve-$k$ then generate)~\citep{luo2023chatkbqa, glass2022re2g} focus on isolated retrieval and generation tasks. Although they can be setup and achieve data retrieval quickly, they struggle to trace complete evidence chains: dense retrievers, typically trained for pointwise recall and re-ranking, often lack path coverage; chunking heuristics fragment long documents and break discourse; long-context prompting shifts budget toward tokens irrelevant to the final answer and provides no explicit \emph{sufficiency} signal.

$\mathbf{C_2}:\textbf{Uncontrolled iteration and latency}$: With multi-agent collaboration and reasoning, agentic systems~\citep{liu2025hm, yang2025beyond, chen2025improving} easily explode the search space and can achieve multi-step reasoning. However they may fall with branchy plans, repeated web/file calls, and verbose chain-of-thought prompts, yielding unpredictable token/tool costs and latency; termination is often heuristic, leading to premature answers or extra wasted loops with budgets decoupled from the \emph{evidence actually inspected}~\citep{singh2025agentic}.

$\mathbf{C_3}:\textbf{Heterogeneity across formats}$: Free text, relational tables, and KGs typically require distinct indices, retrievers, prompt styles, and controller logic, preventing policy reuse and complicating training and deployment. Although existing heterogeneous RAG systems~\citep{yu2022retrieval, christmann2024rag} are available to deal with multiple formats of data, they may still face issues in either weak alignment across representations or lossy and non-reversible serialization that obscures provenance and blocks faithful reconstruction.


To address above issues, this paper introduces \textbf{RELOOP}, a \textbf{Hierarchical Sequence Iteration System } that first recasts heterogeneous knowledge source into a \emph{single, LLM-native interface}, then turning retrieval into a \emph{guided, budget-aware iterative process}. The reversible structured parsing linearizes documents, tables, and KGs into a sequence of typed segments with lightweight structure (e.g., parent/child locality, offsets or coordinates, minimal schema/time tags).  An iteration policy operates on this unified substrate using short, budgeted steps: at each step it selects a few promising segments and predicts whether the accumulated set is sufficient to answer. A concise \emph{guidance} plan—produced by a lightweight planner or a heuristic template—acts as a soft prior over which regions to probe first and when to stop. Once sufficiency is predicted, the selected segments are canonicalized into a compact, provenance-preserving package consumed by a head module to produce the final answer; an optional verifier can trigger a brief refinement if contradictions are detected. pecifically, our \textbf{key contributions} are as followed:

\begin{itemize}
\item \textbf{Unified, reversible structured parsing.} A hierarchical sequence representation\textbf{(HSEQ)} that standardizes text, tables, and KGs with lightweight structure and source, enabling a single parser to operate across formats.
\item \textbf{Guided, budget-aware iteration.} A learned selection policy with an explicit sufficiency signal that concentrates computation on \emph{evidence actually inspected}, delivering predictable latency under token/tool budgets.
\item \textbf{Canonicalized evidence for reliable QA.} A compact, provenance-preserving evidence package that improves answer synthesis and auditability, and supports optional contradiction-driven refinement.
\end{itemize}

\section{Related Work}

\paragraph{LLM Finetuning}

Large Language Models (LLMs) often adopt finetuning to unlock their capabilities for downstream applications, like medical~\citep{goyal2024healai}, economic~\cite{guo2024econnli}, or human activity recognition~\cite{li2024sensorllm}. To enhance finetuning efficiency, methods like quantization~\citep{dettmers2022gpt3} parameter efficient fine tuning~\citep{hu2022lora, dettmers2023qlora, li2021prefix} can be applied.

\paragraph{Retrieval Augmented Generation}

RAG systems help LLMs retrieve extra knowledge according to queries and thereby improving the accuracy of LLM response~\citep{fan2024survey}, with no necessity to finetune the model. External databases ensure knowledge offered is domain-specific and timely, adding reliability and interpretability~\citep{lewis2020retrieval, jiang2023active}. \textbf{Accuracy} of knowledge retrieval and \textbf{quality} of responses are two key factors for RAG systems evaluation~\citep{yu2024evaluation}. Apart from text, table, or html sources~\citep{guo2024lightrag, chan2024rq, jin2025flashrag}, recent researches have combined graph-structured data into RAG systems(GraphRAG) to improve the efficiency of knowledge interpretability by capturing relationships between entities and utilizing triplets as the primary data source~\citep{edge2024local,peng2024graph,hu2024grag, mavromatis2024gnn}. 

\paragraph{Multi Agent QA system}

LLM-based Multi-Agent Systems (MASs) enable groups of intelligent agents to coordinate and solve complex tasks collectively at scale, transitioning from isolated models to collaboration-centric approaches~\citep{tran2025multi}. Agents can cooperate with each other for tasks like code generation~\citep{hong2024metagpt, islam2024mapcoder}, decision making~\citep{nascimento2023self, shinn2023reflexion}, while competitions among agents are appiled on gaming environment~\cite{wang2022cooperative} or question answering~\citep{puerto2021metaqa}. By interacting with each other, the system can be used for both problem solving or world simulation~\citep{guo2024large}

\paragraph{Structural and unified RAG interfaces.}
Beyond standard text-centric RAG, a line of work introduces \emph{structural} or \emph{unified} retrieval layers. Graph-based RAG systems construct heterogeneous or chunk-level graphs where nodes represent passages, entities, or sections, and edges encode semantic or hyperlink connectivity; retrieval then propagates over this graph to improve multi-hop reasoning and global coverage~\citep{wu2024medical, huang2025ket, luo2025gfm}. Other systems build hierarchical or modular indices over mixed document formats, or define unified data schemas for training language agents and tools, but still operate over opaque contexts at inference time~\citep{reynolds2024improving, liu2025hm, chen2024hiqa}. 


While existing RAG-based methods still suffered from limitation mentioned above, there is a rising need for RAG interfaces that (i) preserve modality-specific structure in a \emph{reversible} way rather than collapsing all sources into an opaque graph or index; (ii) expose a generic, LLM-native segment schema with explicit level, parent, and alignment fields so that a single controller can navigate text, tables, and KGs uniformly; and (iii) couple this interface with an explicit sufficiency-aware, budget-controlled selection process, so that evidence gathering is both auditable and aligned with resource constraints. \textbf{RELOOP} is designed to meet these requirements by treating all sources as typed, provenance-aware segments in a hierarchical sequence and pairing this representation with a learned sufficiency head and budget-aware iteration policy.

\section{RELOOP: A Multi-Agent Heterogeneous Question Answering Framework}
\label{sec:framework-overview}

\begin{figure*}[t]
\centering
\includegraphics[width=\textwidth]{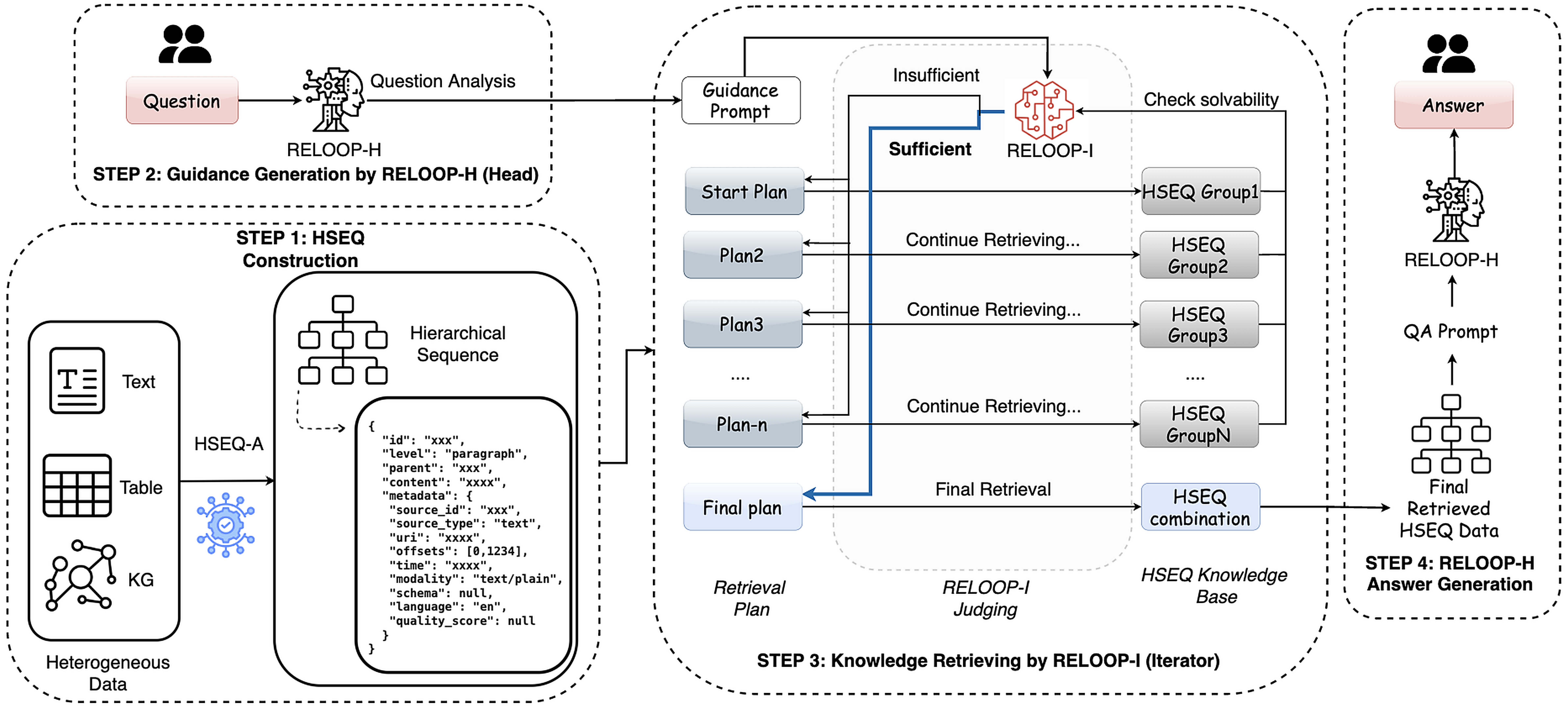}
\caption{RELOOP overview. (\emph{i}) HSEQ-A linearizes heterogeneous sources into HSEQ $S_h$ with level tags, parent pointers, and standardized metadata; (\emph{ii}) RELOOP-I iterates over a windowed stream of segments under budgets, guided by $g$, and queries $\Phi$ for sufficiency; (\emph{iii}) $\kappa$ compacts $M_t$ into provenance-preserving evidence; (\emph{iv}) RELOOP-H produces the final answer and optionally triggers a brief refinement if inconsistencies are detected.}
\label{fig:hseq-framework}
\end{figure*}

\subsection{Background and Setup}
\label{sec:background-setup}

\paragraph{Heterogeneous QA with budgets.}
Given a natural-language query $q$ and a heterogeneous corpus $D=\{(x_j,m_j)\}_{j=1}^N$ with modality $m_j\in\{\texttt{text},\texttt{table},\texttt{kg}\}$, the goal is to produce an answer $y\in\mathcal{Y}$ and optional supporting evidence $E\subseteq D$ while satisfying resource budgets $B$ (tokens, tool calls, latency, steps). Let $E^\star$ denote a \emph{minimally sufficient} evidence set for $q$ in $D$ under a fixed answerer.

\paragraph{From retrieval to guided iteration.}
We recast retrieval as a short sequence of structure-aware selections under an explicit sufficiency criterion. A modality-aware adapter $\tau$ converts $D$ into a single hierarchical sequence $S_h=\tau(D)$. A learned iteration policy $\pi_\theta$ interacts with $(q,S_h)$ to accumulate a compact evidence set $M^\star$ under budgets $B$, guided by a concise plan $g$. A canonicalizer $\kappa$ packages $M^\star$ for a head module $\mathcal{H}$, which produces the final answer. This preserves the familiar RAG workflow while adding a principled stopping signal and a unified interface across modalities.

\subsection{RELOOP Architecture}
The proposed system couples a unified \emph{hierarchical sequence (HSEQ)}
representation with an iteration policy and a head module $\mathcal{H}$ for
answer synthesis. As described in Section~\ref{sec:background-setup}, a
modality-aware adapter $\tau$ converts $D$ into a single hierarchical sequence
$S_h$, and the iterator $\pi_\theta$ accumulates evidence under budgets $B$, with framework as:
\begin{equation}
\label{eq:framework}
F = \big(\tau, \pi_\theta, \Phi, \kappa, \mathcal{H}\big),\quad
F(q,D) \;=\; \mathcal{H}\!\left(q,\, \kappa\!\left(M^\star\right)\right),
\end{equation}
where $\Phi$ denotes a budget-aware sufficiency criterion used by $\pi_\theta$,
and $\kappa$ canonicalizes the final evidence $M^\star$ for $\mathcal{H}$.

During iteration, $\pi_\theta$ selects and expands segments on $S_h$ under the
budget state $B_t$ and stops when evidence is deemed sufficient by $\Phi$ or the
budget is exhausted:
\begin{equation}
\label{eq:iter}
S_h = \tau(D),\qquad
M^\star \;=\; \pi_\theta\big(q,S_h;\,\Phi,\,B\big).
\end{equation}

After the iteration, $\kappa$ maps raw segments $M^\star$ to a normalized evidence package consumable by $\mathcal{H}$. The same policy $\pi_\theta$ is shared across modalities due to the common interface of $S_h$.

Generally, to achieve iteration through an unified data structure building from heterogeneous data sources, the RELOOP framework consists of three key modules: HSEQ-Adapter (HSEQ-A), RELOOP-Iterator (RELOOP-I), and RELOOP-Head (RELOOP-H).

\subsection{HSEQ-Adapter(HSEQ-A)}
\label{sec:hseq-a}

The HSEQ-Adapter is build to produce unified structure(\emph{HSEQ} $S_h$) that exposes locality (parent/child), alignment (span or coordinates), and lightweight semantics (time, schema, language) in a modality-agnostic format, while remaining \emph{reconstructable}. Formally, each item $x_j$ is mapped by a modality-specific adapter $\tau_{m_j}$ to a finite set of segments $\tau_{m_j}(x_j)\subset\mathcal{S}$ and then concatenated as 
\( S_h = \bigsqcup_{j=1}^N \tau_{m_j}(x_j) \in \mathcal{S}^\ast \),
where each \( s \in \mathcal{S} \) is represented as
\( s = (\mathrm{id}(s), \ell(s), p(s), c(s), \mu(s)) \).

$\ell(s)$ is a level tag matching the raw content, including sentence, paragraph, table, triplet, etc., while $p(s)$ is a parent pointer recording the roots. $c(s)$ is human-readable content, and $\mu(s)$ is metadata with fixed keys to record content attributes.



Segments are concatenated into a final usable $S_h$ in parent-before-child order.
This minimal contract enables structure-aware neighborhoods and budget-aware iteration without inspecting raw files.

\subsection{RELOOP-Iterator(RELOOP-I)}
\label{sec:hseq-i}


Following details how the iterator $\pi_\theta$ is guided by $g$ and controlled by windowing, neighborhood expansion, and a sufficiency-aware stopping signal.

\paragraph{Guidance prior.}
A short guidance $g = g(q,\mathrm{type})$ is treated as a \emph{prior} over iterator actions. $g$ is generated before each episode to shape exploration on $S_h$. This guidance can come from directly from head agent $\mathcal{H}$, or from heuristic templates keyed by $\mathrm{type}$.


\paragraph{Iteration control.}
Let $M_t\subseteq S_h$ denote the selected evidence at step $t$, $C_t\subseteq S_h$ a candidate window obeying a budget state $B_t$, and $\mathcal{N}(\cdot)$ the structure-aware neighborhood operators induced by levels, parents, and coordinates. The RELOOP-I module $\pi_\theta$ functions each step following the policy
\[
\pi_\theta(a_t, s_t \mid q,\,S_h,\,M_t,\,C_t,\,g,\,B_t),
\]
which emits an action $a_t$ (e.g., selecting up to $k$ segments from $C_t$ and/or expanding via $\mathcal{N}$) and a sufficiency prediction $s_t\in\{0,1\}$. A deterministic ordering $\rho$ over $S_h$ (e.g., paragraph $\prec$ row $\prec$ sentence $\prec$ triplet) defines the stream exposed to the policy. State evolves via a deterministic update $M_{t+1}=u(M_t,a_t)$ and $C_{t+1}=\mathrm{window}(S_h,M_{t+1},B_{t+1}, \rho)$. Termination occurs at $\tau=\min\{t:\,s_t=1\}$ or when the budget is exhausted.

With set window size $W$ and step cap $T_{\max}$, the algorithm can be described as Alg.~\ref{alg:guided-selection}, where the \emph{Refresh} operator advances the window while removing already selected segments, keeping the per-step context bounded independent of corpus size.

\paragraph{One-step worked example.}
Suppose $\rho$ is \texttt{paragraph} $\prec$ \texttt{row} $\prec$ \texttt{sentence} $\prec$ \texttt{triplet} and $W\!=\!5$. At $t\!=\!0$, $C_0$ is the first $5$ segments under $\rho$. The policy selects $K_1\!\subseteq\!C_0$ ($|K_1|\!\le\!k$) and optionally expands with $\mathcal{N}_{\text{children}}$ (to get sentences within a paragraph) or $\mathcal{N}_{\text{row}}$ (to fetch a full table row when a cell is promising). Then $M_1\!=\!M_0\cup K_1$. $\mathrm{Refresh}$ advances the window to the next $5$ \emph{unseen} segments in $S_h$. If $\Phi$ deems evidence sufficient ($s_1\!=\!1$) and $t\!\ge\!T_{\min}$, iteration halts; otherwise proceed to $t\!=\!2$ with updated $B_t$.

\subsection{RELOOP-Head (RELOOP-H).}
\label{sec:hseq-h}

The RELOOP-Head module $\mathcal{H}$ can be used in two parts: 1) Guiding the retrieval for RELOOP-I; and 2) Generating final conclusion regarding the question.

\paragraph{Guidance Generation.}
Although heuristic templates can be used, regarding an incoming question, $\mathcal{H}$ is available to be called first to analysis the content, generating guidance including: 1) Initial Retrieval Plan; 2) What information may be needed; 3) Potential conditions to stop.

\paragraph{Answer synthesis and optional refinement.}
Upon termination at step $\tau$, the canonicalizer $\kappa$ converts $M_\tau$ into a compact, provenance-preserving evidence package(ids, levels, offsets/coordinates, short snippets). The head module $\mathcal{H}$ then produces the final prediction:
\[
\hat{y}=\mathcal{H}\big(q,\,\kappa(M_\tau)\big).
\]
An optional verifier $\xi$ inspects $\kappa(M_\tau)$ for contradictions; if detected, a brief refinement pass (at most $\Delta$ additional steps) resumes iteration in Alg.~\ref{alg:guided-selection} with tightened guidance $g'$ and reduced budget $B'$.

\begin{algorithm}[H]
\caption{Guided Iterative Selection under RELOOP-I}
\label{alg:guided-selection}
\begin{algorithmic}[1]
\Require question $q$, HSEQ $S_h$, guidance $g$, budget $B$, window size $W$, step cap $T_{\max}$, minimum steps $T_{\min}$, top-$k$ $k$, ordering $\rho$
\State $M_0 \gets \varnothing$;\quad $C_0 \gets \mathrm{Window}(S_h; W, B_0,\rho)$
\For{$t = 1$ to $T_{\max}$}
  \State  $\mathrm{Update}\, a_t$
  \State $(K_t, s_t) \xleftarrow{a_t} \pi_\theta(q, g, M_{t-1}, C_{t-1}; B_t)$ \Comment{$K_t \subseteq C_{t-1},\ |K_t|\le k$}
  \State $M_t \gets M_{t-1} \cup K_t$
  \State $C_t \gets \mathrm{Refresh}(S_h, M_t; W,\rho)$
  \If{$s_t = 1$ \textbf{and} $t \ge T_{\min}$} \State \textbf{break} \EndIf
  \State  $\mathrm{Update}\, B_t$
\EndFor
\State $\tau \gets t$; \quad \Return $\kappa(M_\tau)$
\end{algorithmic}
\end{algorithm}








\section{RELOOP usage with Open-Source LLMs}
\label{sec:learning-usage}
With Section~\ref{sec:framework-overview} as the theoretical interface, this section details how to instantiate RELOOP with open-source LLMs.

\subsection{Fine-tuning RELOOP-I}
\label{sec:finetune}

\paragraph{Training tuples and supervision.}
We organize supervision as tuples $(q,\mathrm{type},S_h,A^\star)$. Besides $q$ and $S_h$, an optional label $\mathrm{type}$ is added. The trajectory $A^\star=\{(a_t^\star,s_t^\star)\}_{t=1}^{\tau^\star}$ contains an action and a binary sufficiency signal with $\tau^\star=\min\{t:\,s_t^\star=1\}$. When explicit trajectories are unavailable, \emph{weak positives} $P^\star\subseteq S_h$ are induced by high-precision matching between gold answers (or oracle spans) and segment content, optionally augmented by lexical overlap with $q$. A target action sequence is synthesized by greedily selecting from $P^\star$ under the budget (details in App.~\ref{app:weakpos}).




\paragraph{Policy learning.}
The policy $\pi_\theta$ is trained with parameter-efficient adaptation of a base
LLM using Low-Rank Adaptation (LoRA)~\citep{hu2022lora}
(App.~\ref{app:lora}).
With teacher forcing for $t < \tau^\star$, the objective minimizes an
action-matching loss $\ell_{\mathrm{act}}$ and a sufficiency-aware stopping loss
$\ell_{\mathrm{stop}}$, as formalized in Eq.~\ref{eq:objective},
where $\mathrm{state}_t = (q, S_h, M_t, C_t, g, B_t)$ and $\lambda > 0$ trades off
action accuracy and early stopping.
When $A^\star$ is synthesized from $P^\star$, per-step weights attenuate
low-confidence choices to reduce label noise (App.~\ref{app:weakpos}).

\begin{equation}
\begin{aligned}
\min_{\theta}\quad 
\mathbb{E}\Bigg[
\sum_{t=1}^{\tau^\star}
&\;\ell_{\mathrm{act}}\!\left(
\pi_\theta(\cdot\mid \mathrm{state}_t),\, a_t^\star
\right) \\
&\;+\;
\lambda\,
\ell_{\mathrm{stop}}\!\left(
\pi_\theta(\cdot\mid \mathrm{state}_t),\, s_t^\star
\right)
\Bigg].
\end{aligned}
\label{eq:objective}
\end{equation}

\begin{figure}[t]
  \includegraphics[width=\columnwidth]{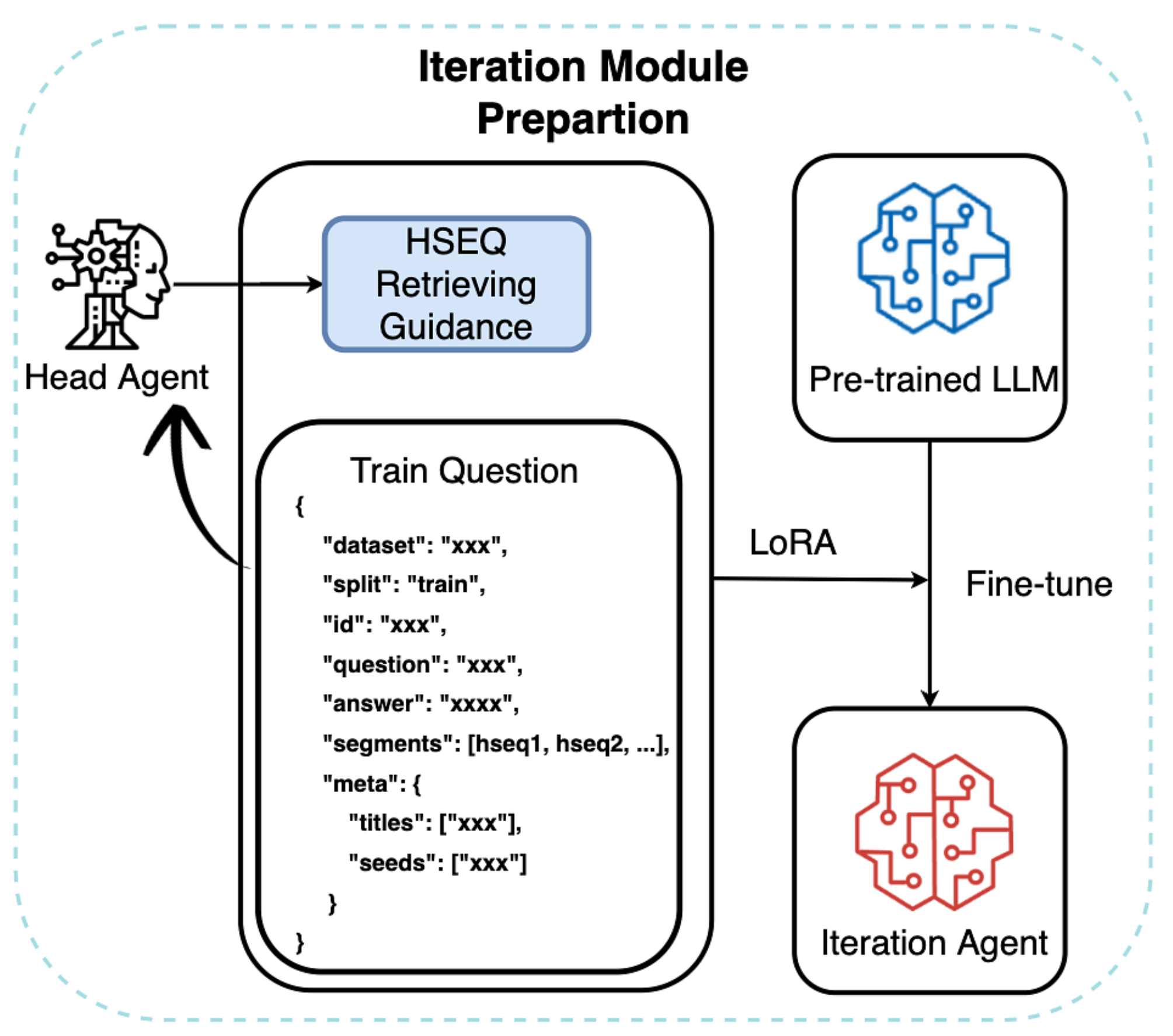}
  \caption{RELOOP-I is trained from multi-source questions. After guidance sets are prepared, LoRA is applied for finetuning.}
  \label{fig:hseq-i-train}
\end{figure}

\paragraph{Guidance generation (RELOOP-H).}
Given a question $q$ (and optional $\mathrm{type}$), RELOOP-H produces a short guidance $g$ that steers the iterator and specifies a stop rule. We use two modes: (i) a lightweight planner that drafts $g$ in 2–4 sentences; (ii) a heuristic template keyed by coarse question patterns (e.g., number/factoid/yes–no). $g$ follows a fixed structure: \emph{first-look targets} (entities/rows/1–2-hop neighbors), \emph{expansion rule} (parent/child, row/column, or relation hops), and \emph{stop rule} (e.g., “answer span/number is explicit and corroborated by $\ge1$ snippet”). $g$ is cached and reused; on a cache miss, we generate or fall back to the template. $g$ is a \emph{soft prior}—the iterator may override it when stronger signals appear.

\paragraph{Evidence-conditioned answering (RELOOP-H).}
After $\kappa$ compacts $M_\tau$ to $Z=\kappa(M_\tau)$ (snippets plus ids/levels/source and minimal offsets/coordinates), RELOOP-H performs evidence-conditioned answering:$\hat{y}\;=\;\mathcal{H}\!\big(q,\,Z;\,g\big)$
using a minimal prompt: \emph{answer only} (span/number/yes–no) grounded in $Z$, no chain-of-thought. When useful, we also return supporting ids from $Z$ for auditability. A lightweight entailment-style check over $Z$ may trigger a one-shot \emph{refinement}—the iterator resumes for a few steps under a tightened $g'$—otherwise $\hat{y}$ is emitted.

\section{Experiment}
RELOOP are evaluated on multiple QA datasets with a focus on both answer quality and efficiency. Metrics include Accuracy, F1, alongside efficiency indicators that reflect the \emph{evidence actually inspected}: average iteration steps and end-to-end latency (ms). 

\subsection{Experiment Setup: Benchmarks and Baselines}

\begin{table}
\vspace{-6pt}
\centering
\caption{Datasets used in our study (modality abbreviations: T=Text, Tb=Table, KG=Knowledge Graph). }
\label{tab:datasets}
\resizebox{\columnwidth}{!}{
\begin{tabular}{lcccc}
\toprule
\textbf{Dataset} & \textbf{Modality} & \textbf{\#Train} & \textbf{\#Validation}  & \textbf{\#Test} \\
\midrule
HotpotQA & T & 90447 & 7405 & 7405 \\
TAT-QA   & Tb + T & 13,251 & 1644 & 1,663 \\
HybridQA & Tb + T & 62,682 & 3,466 & 3463  \\
MetaQA-2Hop   & KG & 119,986 & 14,872 & 14,872  \\
MetaQA-3Hop   & KG & 17,482  & 14,274 & 14,274  \\
\bottomrule
\end{tabular}
}
\vspace{-4pt}
\end{table}

\paragraph{Benchmarks.}
To evaluate RELOOP usage from different data modalities, four benchmarks are used for experiments, stressing text-only, table–text hybrid, and KG reasoning: 
\emph{HotpotQA~\citep{yang2018hotpotqa}} (multi-hop reasoning over Wikipedia text), 
\emph{TAT-QA~\citep{zhu2021tat}} (table-centric financial QA with accompanying paragraphs and numerical operations), 
\emph{HybridQA~\citep{chen2020hybridqa}} (Wikipedia tables linked to passages requiring cross-format grounding), and 
\emph{MetaQA~\citep{zhang2018variational}} over a Wikidata-style KG (Since 1-hop variants are not emphasized due to triviality, during experiment only 2-hop and 3-hop questions are used for experiments).

\paragraph{Baselines.}
Three groups are divided as:
\begin{itemize}
\item \textbf{LLM-only QA.} Multiple LLMs are used to directly answers each question from raw knowledge source inputs, under same prompt instruction.

\item \textbf{RAG-based methods.} Since RELOOP explores different formats of data sources, RAG models specializing in separately \emph{Text}, \emph{Table} and \emph{Knowledge Graphs} have been tested. 

Specifically, for HybridQA and TAT-QA, \emph{TAT-LLM}~\citep{zhu2024tat}, \emph{TableRAG}~\citep{yu2025tablerag}, \emph{ODYSSEY}~\citep{agarwal2025hybrid}, \emph{TTQA-RS}~\citep{bardhan2024ttqa} and \emph{HippoRAG}~\citep{jimenez2024hipporag} are chosen for comparison. While for HotpotQA and MetaQA-2hop and 3hop, graph-centric RAG systems \emph{Graph-constrained Reasoning(GcR)}~\citep{luo2024graph}, \emph{Think on Graph (ToG)}~\citep{ma2024think} and \emph{AdaptiveRAG}~\citep{jeong2024adaptive} are set as baselines. Each is configured per its recommended settings for the corresponding modality.

\item \textbf{RELOOP (ours).} (i) The \emph{best} iteration–head pair results and (ii) The \emph{median} pair results over a grid of open-source models are provided. Three ablations are also included in experiments: (i) \emph{LLM-only }; (ii) \emph{RELOOP w/o SFT} (RELOOP-I not fine-tuned) (iii) \emph{without guidance from RELOOP-H}and (iii) \emph{with only guidance from template}.
\end{itemize}

\paragraph{RELOOP variants.}
For the \emph{iteration agent} (RELOOP-I) and the \emph{head agent} (RELOOP-H), different LLMs are finetuned and used, listed as:

\begin{tabularx}{\linewidth}{lX}
RELOOP-I: &
Falcon-H1-3B-Instruct; Qwen3-4B-Instruct-2507; 
DeepSeek-R1-Distill-Qwen-7B; Falcon3-3B-instruct; 
Falcon3-7B-instruct; Llama-3.2-3B-Instruct. \\
RELOOP-H: &
Falcon3-10B-instruct; Falcon-H1-7B-Instruct; 
Llama-3.1-8B-Instruct; DeepSeek-R1-Distill-Qwen-7B. \\
\end{tabularx}

Compatible pairs are swept and final “best” and “median” results across benchmarks are counted, with hyperparameters settings listed in App.~\ref{app:impl}.

\begin{table*}[h]
\centering
\caption{Overall QA performance on heterogeneous benchmarks. Shaded cells (\textsc{N/A}) indicate the method is not applicable to that benchmark; gray dashes (--) indicate metric not reported. The record results use Qwen3-4B-Instruct-2507 for HSEQ-I; and Falcon-H1-7B-Instruct for HSEQ-H}
\label{tab:main-results}
\resizebox{\textwidth}{!}{
\begin{tabular}{|l|cc|cc|cc|cc|cc|}
\hline
\multirow{2}{*}{Method} & \multicolumn{2}{c|}{HybridQA} & \multicolumn{2}{c|}{TAT-QA} & \multicolumn{2}{c|}{HotpotQA} & \multicolumn{2}{c|}{MetaQA-2hop} & \multicolumn{2}{c|}{MetaQA-3hop} \\
 & Acc & F1 & Acc & F1 & Acc & F1 & Acc & F1 & Acc & F1 \\
\hline
\rowcolor{black!10}\multicolumn{11}{|l|}{\textbf{LLM-only (direct QA)}}\\
Falcon3-10B-instruct     & 22.4 & \NR & 35.2 & \NR & 16.5 & \NR & 43.0 & \NR & 39.8 & \NR \\
Falcon-H1-7B-Instruct    & 32.9 & \NR & 43.7 & \NR & 21.1 & \NR & 48.3 & \NR & 44.6 & \NR \\
Llama-3.1-8B-Instruct    & 28.1 & \NR & 37.6 & \NR & 14.6 & \NR & 37.8 & \NR & 31.9 & \NR \\
Qwen3-4B-Instruct-2507   & 30.3 & \NR & 42.1 & \NR & 17.8 & \NR & 42.2 & \NR & 38.5 & \NR \\
\hline
\rowcolor{black!10}\multicolumn{11}{|l|}{\textbf{RAG-based methods (single-pass / agentic baselines)}}\\
TAT-LLM (Table task)                 & \NA  & \NA  & \underline{73.1} & \underline{81.0} & \NA  & \NA  & \NA  & \NA  & \NA  & \NA  \\
TableRAG   (Table task)                & 47.9 & \NR  & 61.9 & 68.6 & \NA  & \NA  & \NA  & \NA  & \NA  & \NA  \\
ODYSSEY    (Text task)             & 51.5 & 66.0 & \NA  & \NA  & \NA  & \NA  & \NA  & \NA  & \NA  & \NA  \\
TTQA-RS    (Text task)         & 62.3 & 70.6 & \NA  & \NA  & \NA  & \NA  & \NA  & \NA  & \NA  & \NA  \\
HippoRAG  (Table + text)              & \underline{65.8} & \textbf{72.4} & 70.1 & 74.9 & 53.2 & 55.7 & \NA  & \NA  & \NA  & \NA  \\
Graph-constrained Reasoning (GcR)
                         & \NA  & \NA  & \NA  & \NA  & 39.2 & 41.6 & 86.7 & 88.1 & 83.2 & 80.6 \\
Think on Graph (ToG)     & \NA  & \NA  & \NA  & \NA  & 43.1 & 44.7 & 83.2 & 84.8 & 81.1 & 78.5 \\
AdaptiveRAG  (Graph task)            & \NA  & \NA  & \NA  & \NA  & 50.3 & 52.5 & 88.2 & 90.1 & 84.5 & 85.7 \\
\hline
\rowcolor{black!10}\multicolumn{11}{|l|}{\textbf{Our method: RELOOP}}\\
RELOOP (best)              & \textbf{66.4} & \underline{72.1} & \textbf{75.7} & \textbf{83.5} & \textbf{56.3} & \textbf{58.6} & \textbf{95.9} & \textbf{91.1} & \textbf{93.4} & \textbf{88.3} \\
RELOOP (median)            & 63.9 & 70.8 & 73.2 & 79.6 & \underline{55.4} & \underline{57.1} & \underline{93.2} & \underline{89.7} & \underline{90.1} & \underline{86.6} \\
\hline
\end{tabular}
}
\end{table*}

\subsection{Experiment Result: How competitive is RELOOP with other baselines?}
Table~\ref{tab:main-results} summarizes answer quality across all datasets. RELOOP consistently improves over both LLM-only and strong RAG baselines, while using controlled iteration and exposing explicit provenance. Detailed per-model pairs are reported in Table~\ref{tab:acc-f1-eff}. Efficiency measurements (tokens/latency/steps) are in Table~\ref{tab:efficiency}.

Our RELOOP achieves strong and consistent gains on multiple benchmarks. On HotpotQA, MetaQA-2hop, and MetaQA-3hop, both the \emph{best} and \emph{median} RELOOP configurations surpass all baselines. On TAT-QA, RELOOP’s best run attains the top score overall, while the median run trails slightly behind TAT-LLM~\citep{zhu2024tat}. On the table-and-text HybridQA, RELOOP attains the best accuracy and the second-best F1 (just behind HippoRAG~\citep{jimenez2024hipporag}); the median configuration remains third in baselines.

\subsection{Yielding between efficiency and accuracy}

Table~\ref{tab:acc-f1-eff} in App.~\ref{app:var-llm-res} lists results using different RELOOP-I and HSEQ-A. The HybridQA results reveal a clear accuracy–efficiency trade-off across RELOOP agent pairs. The highest accuracy/F1 comes from Qwen3-4B (RELOOP-I) + Falcon-H1-7B (RELOOP-H) (66.2 / 71.4), with the \underline{second-best} Qwen3-4B + Llama-3.1-8B (65.5 / 71.2). These configurations, however, incur larger iteration depth and latency (about 3.7–4.1 steps; 16.5–21.5 second). On the efficiency end, Llama-3.2-3B + Llama-3.1-8B delivers the lowest steps and latency (2.11; 8.35k ms) with moderate accuracy (55.4 / 57.9), while Falcon3-3B + Falcon-H1-7B attains the \underline{second-best} efficiency (2.25; 11.7k ms) at similar quality. Taken together, the Pareto frontier spans (i) Qwen-based iterators with larger heads for top accuracy, and (ii) lightweight Llama/Falcon pairs for predictable low latency. Different agent pairs can be chosen regarding whether accuracy or budget dominates.

\subsection{Efficiency Analysis}
To test RELOOP framework's latency, \emph{evidence actually inspected} are calculated: iteration steps for RELOOP-I and wall-clock latency are calculated. Results are summarized below. “LLM-only” incurs a single forward pass (1 step) and thus the lowest raw latency,  but this comes at the cost of weaker multi-hop accuracy and no explicit provenance in Table~\ref{tab:acc-f1-eff}.  In contrast, graph-centric ToG performs many expansion steps (11–17 on average), which substantially increases latency (e.g., over $22$k ms on HotpotQA and $24$k ms on MetaQA-3hop), even though it is designed for multi-hop reasoning.


\begin{table*}[h]
\centering
\caption{Efficiency metrics on HotpotQA, MetaQA-2hop and MetaQA-3hop.}
\label{tab:efficiency}
\resizebox{\textwidth}{!}{
\begin{tabular}{|l|cc|cc|cc|}
\hline
\rowcolor{black!10}\multicolumn{7}{|l|}{\textbf{Efficiency}}\\
\hline
\multicolumn{1}{|c|}{Method} &
\multicolumn{2}{c|}{HotpotQA} &
\multicolumn{2}{c|}{MetaQA-2hop} &
\multicolumn{2}{c|}{MetaQA-3hop} \\
\cline{2-7}
& \multicolumn{1}{c|}{Steps} & \multicolumn{1}{c|}{Latency (ms) $\downarrow$}
& \multicolumn{1}{c|}{Steps} & \multicolumn{1}{c|}{Latency (ms) $\downarrow$}
& \multicolumn{1}{c|}{Steps} & \multicolumn{1}{c|}{Latency (ms) $\downarrow$} \\
\hline
LLM-only   & 1 & 3266.3 & 1 & 2556.4 & 1 & 3631.1 \\
Think on Graph (ToG)      & 13.28 & 22708.2 & 11.73 & 15707.6 & 16.58 & 24307.4    \\
RELOOP (ours, best)      & 4.00 & 6247.0 & 3.27 & 5732.2 & 4.11 & 10632.8  \\
RELOOP (ours, median)    & 4.17 & 12114.4 & 3.76 & 9480.1 & 4.59 & 13505.3 \\
\hline
\end{tabular}
}
\end{table*}

\begin{table*}[ht]
\centering
\caption{Ablations on benchmarks.}
\label{tab:ablation}
\resizebox{\textwidth}{!}{
\begin{tabular}{|l|cc|cc|cc|cc|cc|}
\hline
Variant & \multicolumn{2}{c|}{HybridQA} & \multicolumn{2}{c|}{TAT-QA} & \multicolumn{2}{c|}{HotpotQA} & \multicolumn{2}{c|}{MetaQA-3hop} & \multicolumn{2}{c|}{MetaQA-3hop} \\
 & Acc & F1 & Acc & F1 & Acc & F1 & Acc & F1 & Acc & F1\\
\hline
RELOOP (full) & \textbf{66.4} & \textbf{72.1} & \textbf{75.7} & \textbf{83.5} & \textbf{56.3} & \textbf{58.6} & \textbf{95.9} & \textbf{91.1} & \textbf{93.4} & \textbf{88.3} \\
\quad w/o SFT (base iteration) & 57.3 & 65.7 & 60.4 & 66.9 & 46.5 & 47.8 & 78.3 & 80.1 & 74.6 & 72.5 \\
\quad w/o guidance & 59.2 & 62.6 & 68.8 & 75.1 & 50.5 & 51.2 & 82.4 & 83.0 & 79.2 & 73.8 \\
\quad heuristic-only guidance & \underline{63.8} & \underline{67.3} & \underline{70.4} & \underline{79.9} & \underline{54.7} & \underline{56.1} & \underline{87.3} & \underline{85.4} & \underline{83.9} & \underline{86.1} \\
LLM-only & 32.9 & \NR & 43.7 & \NR & 21.1 & \NR & 48.3 & \NR & 44.6 & \NR \\
\hline
\end{tabular}
}
\end{table*}

RELOOP occupies a middle ground in this trade-off. Both the best and median RELOOP variants maintain short, budgeted loops of roughly $3$–$5$ steps across datasets, yet reduce latency by more than half relative to ToG on all three benchmarks. This indicates that guided, windowed iteration over RELOOP can retain multi-hop capability while avoiding the long expansion chains and repeated graph traversals of ToG. Compared with LLM-only, RELOOP pays a moderate overhead in latency but gains structured evidence and substantially higher accuracy on multi-step questions. RELOOP provides a more balanced operating point with bounded steps and competitive performance.

\subsection{Ablation Studies}
Ablation studies are set to evaluate each component of RELOOP framework on representative text (HotpotQA) and table-text (HybridQA) tasks. Following tasks are considered: (a) \textbf{No SFT} (RELOOP-I not fine-tuned); (b) \textbf{No guidance} (remove $g$); (c) \textbf{Heuristic-only guidance} (no planner) ; and (d) \textbf{LLM-only} (without multi-agent but use HSEQ as data input).

The ablation study demonstrates the necessities of all RELOOP's components, with differing sensitivity across formats. Using \emph{heuristic-only} guidance yields the smallest degradation from the full system—typically a modest drop in Acc/F1—indicating that a lightweight, template-style prior already guides RELOOP-I effectively when the planner is absent. Removing fine-tuning (\emph{w/o SFT}) causes a larger decline, but with the use of structured HSEQ data, accuracy remains substantially higher than \emph{LLM-only}. Without guidance (\emph{w/o guidance}) influence performance, as in prompt RELOOP-I is only asked to \textit{choose necessary evidence from below to answer the question}. The results underscore the role of guidance as a portable sufficiency prior. Finally, the \emph{LLM-only} setting performs worst across all benchmarks, reflecting the difficulty of recovering minimally sufficient evidence without iterative, structure-aware selection. Overall, the results suggest that (i) RELOOP’s unified data structure is the primary source of robustness, (ii) SFT RELOOP-I provides consistent gains, and (iii) guidance—even a simple heuristic ones from template-would increase overall accuracy strongly.

\section{Conclusion}

This paper introduces \textbf{RELOOP}, a compact framework for heterogeneous QA that (i) \emph{unifies} text, tables, and knowledge graphs into a reversible hierarchical \textbf{HSEQ} with lightweight structure and provenance; (ii) performs \emph{guided, budget-aware iteration} that selects small sets of salient segments and predicts \emph{sufficiency} for early stopping; and (iii) feeds a \emph{canonicalized evidence} package to a head module for answer synthesis. By replacing single-shot retrieval and unconstrained agentic loops with short, structure-aware selections equipped with an explicit sufficiency signal, RELOOP concentrates computation on \emph{evidence actually inspected}, delivers predictable latency under token/tool budgets, and preserves auditability through provenance-aware canonicalization. 

Across heterogeneous QA benchmarks, RELOOP achieves strong answer quality alongside consistent efficiency, revealing a controllable trade-off between accuracy and cost: larger head with finetuned small iterators achieved both fast and accurate QA. The unified sequence parsing and standardized action schema enable a single learned policy to operate across modalities reliably.

\section*{Limitations}

Although through experiments are done to prove RELOOP's universality under different modalities of data, more improvement can be added on the framework. \emph{Future work} will extend RELOOP to multi-turn/streaming settings with dynamic corpora, mitigate hallucination on sufficiency judge under noisy evidence.


\bibliography{custom}

\begin{thebibliography}{57}
\providecommand{\natexlab}[1]{#1}

\bibitem[{Achiam et~al.(2023)Achiam, Adler, Agarwal, Ahmad, Akkaya, Aleman, Almeida, Altenschmidt, Altman, Anadkat et~al.}]{achiam2023gpt}
Josh Achiam, Steven Adler, Sandhini Agarwal, Lama Ahmad, Ilge Akkaya, Florencia~Leoni Aleman, Diogo Almeida, Janko Altenschmidt, Sam Altman, Shyamal Anadkat, and 1 others. 2023.
\newblock Gpt-4 technical report.
\newblock \emph{arXiv preprint arXiv:2303.08774}.

\bibitem[{Agarwal et~al.(2025)Agarwal, Devaguptapu et~al.}]{agarwal2025hybrid}
Ankush Agarwal, Chaitanya Devaguptapu, and 1 others. 2025.
\newblock Hybrid graphs for table-and-text based question answering using llms.
\newblock \emph{arXiv preprint arXiv:2501.17767}.

\bibitem[{Bardhan et~al.(2024)Bardhan, Xiao, and Wang}]{bardhan2024ttqa}
Jayetri Bardhan, Bushi Xiao, and Daisy~Zhe Wang. 2024.
\newblock Ttqa-rs-a break-down prompting approach for multi-hop table-text question answering with reasoning and summarization.
\newblock \emph{arXiv preprint arXiv:2406.14732}.

\bibitem[{Chan et~al.(2024)Chan, Xu, Yuan, Luo, Xue, Guo, and Fu}]{chan2024rq}
Chi-Min Chan, Chunpu Xu, Ruibin Yuan, Hongyin Luo, Wei Xue, Yike Guo, and Jie Fu. 2024.
\newblock Rq-rag: Learning to refine queries for retrieval augmented generation.
\newblock \emph{arXiv preprint arXiv:2404.00610}.

\bibitem[{Chen et~al.(2020)Chen, Zha, Chen, Xiong, Wang, and Wang}]{chen2020hybridqa}
Wenhu Chen, Hanwen Zha, Zhiyu Chen, Wenhan Xiong, Hong Wang, and William Wang. 2020.
\newblock Hybridqa: A dataset of multi-hop question answering over tabular and textual data.
\newblock \emph{arXiv preprint arXiv:2004.07347}.

\bibitem[{Chen et~al.(2024)Chen, Gao, Song, and Tan}]{chen2024hiqa}
Xinyue Chen, Pengyu Gao, Jiangjiang Song, and Xiaoyang Tan. 2024.
\newblock Hiqa: A hierarchical contextual augmentation rag for multi-documents qa.
\newblock \emph{arXiv preprint arXiv:2402.01767}.

\bibitem[{Chen et~al.(2025)Chen, Yan, Sun, Ma, Zhang, Wang, Yin, Yang, and Mao}]{chen2025improving}
Yiqun Chen, Lingyong Yan, Weiwei Sun, Xinyu Ma, Yi~Zhang, Shuaiqiang Wang, Dawei Yin, Yiming Yang, and Jiaxin Mao. 2025.
\newblock Improving retrieval-augmented generation through multi-agent reinforcement learning.
\newblock \emph{arXiv preprint arXiv:2501.15228}.

\bibitem[{Christmann and Weikum(2024)}]{christmann2024rag}
Philipp Christmann and Gerhard Weikum. 2024.
\newblock Rag-based question answering over heterogeneous data and text.
\newblock \emph{arXiv preprint arXiv:2412.07420}.

\bibitem[{Dettmers et~al.(2022)Dettmers, Lewis, Belkada, and Zettlemoyer}]{dettmers2022gpt3}
Tim Dettmers, Mike Lewis, Younes Belkada, and Luke Zettlemoyer. 2022.
\newblock Gpt3. int8 (): 8-bit matrix multiplication for transformers at scale.
\newblock \emph{Advances in neural information processing systems}, 35:30318--30332.

\bibitem[{Dettmers et~al.(2023)Dettmers, Pagnoni, Holtzman, and Zettlemoyer}]{dettmers2023qlora}
Tim Dettmers, Artidoro Pagnoni, Ari Holtzman, and Luke Zettlemoyer. 2023.
\newblock Qlora: Efficient finetuning of quantized llms.
\newblock \emph{Advances in neural information processing systems}, 36:10088--10115.

\bibitem[{Dubey et~al.(2024)Dubey, Jauhri, Pandey, Kadian, Al-Dahle, Letman, Mathur, Schelten, Yang, Fan et~al.}]{dubey2024llama}
Abhimanyu Dubey, Abhinav Jauhri, Abhinav Pandey, Abhishek Kadian, Ahmad Al-Dahle, Aiesha Letman, Akhil Mathur, Alan Schelten, Amy Yang, Angela Fan, and 1 others. 2024.
\newblock The llama 3 herd of models.
\newblock \emph{arXiv e-prints}, pages arXiv--2407.

\bibitem[{Edge et~al.(2024)Edge, Trinh, Cheng, Bradley, Chao, Mody, Truitt, Metropolitansky, Ness, and Larson}]{edge2024local}
Darren Edge, Ha~Trinh, Newman Cheng, Joshua Bradley, Alex Chao, Apurva Mody, Steven Truitt, Dasha Metropolitansky, Robert~Osazuwa Ness, and Jonathan Larson. 2024.
\newblock From local to global: A graph rag approach to query-focused summarization.
\newblock \emph{arXiv preprint arXiv:2404.16130}.

\bibitem[{Fan et~al.(2024)Fan, Ding, Ning, Wang, Li, Yin, Chua, and Li}]{fan2024survey}
Wenqi Fan, Yujuan Ding, Liangbo Ning, Shijie Wang, Hengyun Li, Dawei Yin, Tat-Seng Chua, and Qing Li. 2024.
\newblock A survey on rag meeting llms: Towards retrieval-augmented large language models.
\newblock In \emph{Proceedings of the 30th ACM SIGKDD Conference on Knowledge Discovery and Data Mining}, pages 6491--6501.

\bibitem[{Gao et~al.(2023)Gao, Xiong, Gao, Jia, Pan, Bi, Dai, Sun, and Wang}]{gao2023retrieval}
Yunfan Gao, Yun Xiong, Xinyu Gao, Kangxiang Jia, Jinliu Pan, Yuxi Bi, Yi~Dai, Jiawei Sun, and Haofen Wang. 2023.
\newblock Retrieval-augmented generation for large language models: A survey.
\newblock \emph{arXiv preprint arXiv:2312.10997}.

\bibitem[{Glass et~al.(2022)Glass, Rossiello, Chowdhury, Naik, Cai, and Gliozzo}]{glass2022re2g}
Michael Glass, Gaetano Rossiello, Md~Faisal~Mahbub Chowdhury, Ankita~Rajaram Naik, Pengshan Cai, and Alfio Gliozzo. 2022.
\newblock Re2g: Retrieve, rerank, generate.
\newblock \emph{arXiv preprint arXiv:2207.06300}.

\bibitem[{Goyal et~al.(2024)Goyal, Rastogi, Rajagopal, Yuan, Zhao, Chintagunta, Naik, and Ward}]{goyal2024healai}
Sagar Goyal, Eti Rastogi, Sree~Prasanna Rajagopal, Dong Yuan, Fen Zhao, Jai Chintagunta, Gautam Naik, and Jeff Ward. 2024.
\newblock Healai: A healthcare llm for effective medical documentation.
\newblock In \emph{Proceedings of the 17th ACM International Conference on Web Search and Data Mining}, pages 1167--1168.

\bibitem[{Guo et~al.(2024{\natexlab{a}})Guo, Chen, Wang, Chang, Pei, Chawla, Wiest, and Zhang}]{guo2024large}
Taicheng Guo, Xiuying Chen, Yaqi Wang, Ruidi Chang, Shichao Pei, Nitesh~V Chawla, Olaf Wiest, and Xiangliang Zhang. 2024{\natexlab{a}}.
\newblock Large language model based multi-agents: A survey of progress and challenges.
\newblock \emph{arXiv preprint arXiv:2402.01680}.

\bibitem[{Guo and Yang(2024)}]{guo2024econnli}
Yue Guo and Yi~Yang. 2024.
\newblock Econnli: evaluating large language models on economics reasoning.
\newblock \emph{arXiv preprint arXiv:2407.01212}.

\bibitem[{Guo et~al.(2024{\natexlab{b}})Guo, Xia, Yu, Ao, and Huang}]{guo2024lightrag}
Zirui Guo, Lianghao Xia, Yanhua Yu, Tu~Ao, and Chao Huang. 2024{\natexlab{b}}.
\newblock Lightrag: Simple and fast retrieval-augmented generation.
\newblock \emph{arXiv preprint arXiv:2410.05779}.

\bibitem[{Hong et~al.(2024)Hong, Zhuge, Chen, Zheng, Cheng, Zhang, Wang, Wang, Yau, Lin et~al.}]{hong2024metagpt}
Sirui Hong, Mingchen Zhuge, Jonathan Chen, Xiawu Zheng, Yuheng Cheng, Ceyao Zhang, Jinlin Wang, Zili Wang, Steven Ka~Shing Yau, Zijuan Lin, and 1 others. 2024.
\newblock Metagpt: Meta programming for a multi-agent collaborative framework.
\newblock International Conference on Learning Representations, ICLR.

\bibitem[{Hu et~al.(2022)Hu, Shen, Wallis, Allen-Zhu, Li, Wang, Wang, Chen et~al.}]{hu2022lora}
Edward~J Hu, Yelong Shen, Phillip Wallis, Zeyuan Allen-Zhu, Yuanzhi Li, Shean Wang, Lu~Wang, Weizhu Chen, and 1 others. 2022.
\newblock Lora: Low-rank adaptation of large language models.
\newblock \emph{ICLR}, 1(2):3.

\bibitem[{Hu et~al.(2024)Hu, Lei, Zhang, Pan, Ling, and Zhao}]{hu2024grag}
Yuntong Hu, Zhihan Lei, Zheng Zhang, Bo~Pan, Chen Ling, and Liang Zhao. 2024.
\newblock Grag: Graph retrieval-augmented generation.
\newblock \emph{arXiv preprint arXiv:2405.16506}.

\bibitem[{Huang et~al.(2025)Huang, Zhang, and Xiao}]{huang2025ket}
Yiqian Huang, Shiqi Zhang, and Xiaokui Xiao. 2025.
\newblock Ket-rag: A cost-efficient multi-granular indexing framework for graph-rag.
\newblock In \emph{Proceedings of the 31st ACM SIGKDD Conference on Knowledge Discovery and Data Mining V. 2}, pages 1003--1012.

\bibitem[{Islam et~al.(2024)Islam, Ali, and Parvez}]{islam2024mapcoder}
Md~Ashraful Islam, Mohammed~Eunus Ali, and Md~Rizwan Parvez. 2024.
\newblock Mapcoder: Multi-agent code generation for competitive problem solving.
\newblock \emph{arXiv preprint arXiv:2405.11403}.

\bibitem[{Jeong et~al.(2024)Jeong, Baek, Cho, Hwang, and Park}]{jeong2024adaptive}
Soyeong Jeong, Jinheon Baek, Sukmin Cho, Sung~Ju Hwang, and Jong~C Park. 2024.
\newblock Adaptive-rag: Learning to adapt retrieval-augmented large language models through question complexity.
\newblock \emph{arXiv preprint arXiv:2403.14403}.

\bibitem[{Jiang et~al.(2023)Jiang, Xu, Gao, Sun, Liu, Dwivedi-Yu, Yang, Callan, and Neubig}]{jiang2023active}
Zhengbao Jiang, Frank~F Xu, Luyu Gao, Zhiqing Sun, Qian Liu, Jane Dwivedi-Yu, Yiming Yang, Jamie Callan, and Graham Neubig. 2023.
\newblock Active retrieval augmented generation.
\newblock In \emph{Proceedings of the 2023 Conference on Empirical Methods in Natural Language Processing}, pages 7969--7992.

\bibitem[{Jimenez~Gutierrez et~al.(2024)Jimenez~Gutierrez, Shu, Gu, Yasunaga, and Su}]{jimenez2024hipporag}
Bernal Jimenez~Gutierrez, Yiheng Shu, Yu~Gu, Michihiro Yasunaga, and Yu~Su. 2024.
\newblock Hipporag: Neurobiologically inspired long-term memory for large language models.
\newblock \emph{Advances in Neural Information Processing Systems}, 37:59532--59569.

\bibitem[{Jin et~al.(2025)Jin, Zhu, Dou, Dong, Yang, Zhang, Zhao, Yang, and Wen}]{jin2025flashrag}
Jiajie Jin, Yutao Zhu, Zhicheng Dou, Guanting Dong, Xinyu Yang, Chenghao Zhang, Tong Zhao, Zhao Yang, and Ji-Rong Wen. 2025.
\newblock Flashrag: A modular toolkit for efficient retrieval-augmented generation research.
\newblock In \emph{Companion Proceedings of the ACM on Web Conference 2025}, pages 737--740.

\bibitem[{Lewis et~al.(2020)Lewis, Perez, Piktus, Petroni, Karpukhin, Goyal, K{\"u}ttler, Lewis, Yih, Rockt{\"a}schel et~al.}]{lewis2020retrieval}
Patrick Lewis, Ethan Perez, Aleksandra Piktus, Fabio Petroni, Vladimir Karpukhin, Naman Goyal, Heinrich K{\"u}ttler, Mike Lewis, Wen-tau Yih, Tim Rockt{\"a}schel, and 1 others. 2020.
\newblock Retrieval-augmented generation for knowledge-intensive nlp tasks.
\newblock \emph{Advances in neural information processing systems}, 33:9459--9474.

\bibitem[{Li and Liang(2021)}]{li2021prefix}
Xiang~Lisa Li and Percy Liang. 2021.
\newblock Prefix-tuning: Optimizing continuous prompts for generation.
\newblock \emph{arXiv preprint arXiv:2101.00190}.

\bibitem[{Li et~al.(2024)Li, Deldari, Chen, Xue, and Salim}]{li2024sensorllm}
Zechen Li, Shohreh Deldari, Linyao Chen, Hao Xue, and Flora~D Salim. 2024.
\newblock Sensorllm: Aligning large language models with motion sensors for human activity recognition.

\bibitem[{Liu et~al.(2025)Liu, Liu, Yao, Liu, Meng, Wang, and Ma}]{liu2025hm}
Pei Liu, Xin Liu, Ruoyu Yao, Junming Liu, Siyuan Meng, Ding Wang, and Jun Ma. 2025.
\newblock Hm-rag: Hierarchical multi-agent multimodal retrieval augmented generation.
\newblock \emph{arXiv preprint arXiv:2504.12330}.

\bibitem[{Luo et~al.(2023)Luo, Tang, Peng, Guo, Zhang, Ma, Dong, Song, Lin, Zhu et~al.}]{luo2023chatkbqa}
Haoran Luo, Zichen Tang, Shiyao Peng, Yikai Guo, Wentai Zhang, Chenghao Ma, Guanting Dong, Meina Song, Wei Lin, Yifan Zhu, and 1 others. 2023.
\newblock Chatkbqa: A generate-then-retrieve framework for knowledge base question answering with fine-tuned large language models.
\newblock \emph{arXiv preprint arXiv:2310.08975}.

\bibitem[{Luo et~al.(2024)Luo, Zhao, Haffari, Li, Gong, and Pan}]{luo2024graph}
Linhao Luo, Zicheng Zhao, Gholamreza Haffari, Yuan-Fang Li, Chen Gong, and Shirui Pan. 2024.
\newblock Graph-constrained reasoning: Faithful reasoning on knowledge graphs with large language models.
\newblock \emph{arXiv preprint arXiv:2410.13080}.

\bibitem[{Luo et~al.(2025)Luo, Zhao, Haffari, Phung, Gong, and Pan}]{luo2025gfm}
Linhao Luo, Zicheng Zhao, Gholamreza Haffari, Dinh Phung, Chen Gong, and Shirui Pan. 2025.
\newblock Gfm-rag: graph foundation model for retrieval augmented generation.
\newblock \emph{arXiv preprint arXiv:2502.01113}.

\bibitem[{Ma et~al.(2024)Ma, Xu, Jiang, Li, Qu, Yang, Mao, and Guo}]{ma2024think}
Shengjie Ma, Chengjin Xu, Xuhui Jiang, Muzhi Li, Huaren Qu, Cehao Yang, Jiaxin Mao, and Jian Guo. 2024.
\newblock Think-on-graph 2.0: Deep and faithful large language model reasoning with knowledge-guided retrieval augmented generation.
\newblock \emph{arXiv preprint arXiv:2407.10805}.

\bibitem[{Mavromatis and Karypis(2024)}]{mavromatis2024gnn}
Costas Mavromatis and George Karypis. 2024.
\newblock Gnn-rag: Graph neural retrieval for large language model reasoning.
\newblock \emph{arXiv preprint arXiv:2405.20139}.

\bibitem[{Nascimento et~al.(2023)Nascimento, Alencar, and Cowan}]{nascimento2023self}
Nathalia Nascimento, Paulo Alencar, and Donald Cowan. 2023.
\newblock Self-adaptive large language model (llm)-based multiagent systems.
\newblock In \emph{2023 IEEE International Conference on Autonomic Computing and Self-Organizing Systems Companion (ACSOS-C)}, pages 104--109. IEEE.

\bibitem[{Peng et~al.(2024)Peng, Zhu, Liu, Bo, Shi, Hong, Zhang, and Tang}]{peng2024graph}
Boci Peng, Yun Zhu, Yongchao Liu, Xiaohe Bo, Haizhou Shi, Chuntao Hong, Yan Zhang, and Siliang Tang. 2024.
\newblock Graph retrieval-augmented generation: A survey.
\newblock \emph{arXiv preprint arXiv:2408.08921}.

\bibitem[{Puerto et~al.(2021)Puerto, {\c{S}}ahin, and Gurevych}]{puerto2021metaqa}
Haritz Puerto, G{\"o}zde~G{\"u}l {\c{S}}ahin, and Iryna Gurevych. 2021.
\newblock Metaqa: Combining expert agents for multi-skill question answering.
\newblock \emph{arXiv preprint arXiv:2112.01922}.

\bibitem[{Reynolds and Corrigan(2024)}]{reynolds2024improving}
Andrew Reynolds and Felix Corrigan. 2024.
\newblock Improving real-time knowledge retrieval in large language models with a dns-style hierarchical query rag.
\newblock \emph{Authorea Preprints}.

\bibitem[{Shinn et~al.(2023)Shinn, Cassano, Gopinath, Narasimhan, and Yao}]{shinn2023reflexion}
Noah Shinn, Federico Cassano, Ashwin Gopinath, Karthik Narasimhan, and Shunyu Yao. 2023.
\newblock Reflexion: Language agents with verbal reinforcement learning.
\newblock \emph{Advances in Neural Information Processing Systems}, 36:8634--8652.

\bibitem[{Singh et~al.(2025)Singh, Ehtesham, Kumar, and Khoei}]{singh2025agentic}
Aditi Singh, Abul Ehtesham, Saket Kumar, and Tala~Talaei Khoei. 2025.
\newblock Agentic retrieval-augmented generation: A survey on agentic rag.
\newblock \emph{arXiv preprint arXiv:2501.09136}.

\bibitem[{Tran et~al.(2025)Tran, Dao, Nguyen, Pham, O'Sullivan, and Nguyen}]{tran2025multi}
Khanh-Tung Tran, Dung Dao, Minh-Duong Nguyen, Quoc-Viet Pham, Barry O'Sullivan, and Hoang~D Nguyen. 2025.
\newblock Multi-agent collaboration mechanisms: A survey of llms.
\newblock \emph{arXiv preprint arXiv:2501.06322}.

\bibitem[{Wang et~al.(2022)Wang, Hong, Wang, Xu, Tang, Han, and Kurths}]{wang2022cooperative}
Jianrui Wang, Yitian Hong, Jiali Wang, Jiapeng Xu, Yang Tang, Qing-Long Han, and J{\"u}rgen Kurths. 2022.
\newblock Cooperative and competitive multi-agent systems: From optimization to games.
\newblock \emph{IEEE/CAA Journal of Automatica Sinica}, 9(5):763--783.

\bibitem[{Wu et~al.(2024)Wu, Zhu, Qi, Chen, Xu, Menolascina, and Grau}]{wu2024medical}
Junde Wu, Jiayuan Zhu, Yunli Qi, Jingkun Chen, Min Xu, Filippo Menolascina, and Vicente Grau. 2024.
\newblock Medical graph rag: Towards safe medical large language model via graph retrieval-augmented generation.
\newblock \emph{arXiv preprint arXiv:2408.04187}.

\bibitem[{Yang et~al.(2025)Yang, Xue, Razzak, Hacid, and Salim}]{yang2025beyond}
Ruiyi Yang, Hao Xue, Imran Razzak, Hakim Hacid, and Flora~D Salim. 2025.
\newblock Beyond single pass, looping through time: Kg-irag with iterative knowledge retrieval.
\newblock \emph{arXiv preprint arXiv:2503.14234}.

\bibitem[{Yang et~al.(2018)Yang, Qi, Zhang, Bengio, Cohen, Salakhutdinov, and Manning}]{yang2018hotpotqa}
Zhilin Yang, Peng Qi, Saizheng Zhang, Yoshua Bengio, William~W Cohen, Ruslan Salakhutdinov, and Christopher~D Manning. 2018.
\newblock Hotpotqa: A dataset for diverse, explainable multi-hop question answering.
\newblock \emph{arXiv preprint arXiv:1809.09600}.

\bibitem[{Yu et~al.(2024)Yu, Gan, Zhang, Tong, Liu, and Liu}]{yu2024evaluation}
Hao Yu, Aoran Gan, Kai Zhang, Shiwei Tong, Qi~Liu, and Zhaofeng Liu. 2024.
\newblock Evaluation of retrieval-augmented generation: A survey.
\newblock \emph{arXiv preprint arXiv:2405.07437}.

\bibitem[{Yu(2022)}]{yu2022retrieval}
Wenhao Yu. 2022.
\newblock Retrieval-augmented generation across heterogeneous knowledge.
\newblock In \emph{Proceedings of the 2022 conference of the North American chapter of the association for computational linguistics: human language technologies: student research workshop}, pages 52--58.

\bibitem[{Yu et~al.(2025)Yu, Jian, and Chen}]{yu2025tablerag}
Xiaohan Yu, Pu~Jian, and Chong Chen. 2025.
\newblock Tablerag: A retrieval augmented generation framework for heterogeneous document reasoning.
\newblock \emph{arXiv preprint arXiv:2506.10380}.

\bibitem[{Zhang et~al.(2018)Zhang, Dai, Kozareva, Smola, and Song}]{zhang2018variational}
Yuyu Zhang, Hanjun Dai, Zornitsa Kozareva, Alexander Smola, and Le~Song. 2018.
\newblock Variational reasoning for question answering with knowledge graph.
\newblock In \emph{Proceedings of the AAAI conference on artificial intelligence}, volume~32.

\bibitem[{Zhao et~al.(2024)Zhao, Zhang, Yu, Wang, Geng, Fu, Yang, Zhang, and Cui}]{zhao2024retrieval}
Penghao Zhao, Hailin Zhang, Qinhan Yu, Zhengren Wang, Yunteng Geng, Fangcheng Fu, Ling Yang, Wentao Zhang, and Bin Cui. 2024.
\newblock Retrieval-augmented generation for ai-generated content: A survey.
\newblock \emph{arXiv preprint arXiv:2402.19473}.

\bibitem[{Zhu et~al.(2021{\natexlab{a}})Zhu, Lei, Huang, Wang, Zhang, Lv, Feng, and Chua}]{zhu2021tat}
Fengbin Zhu, Wenqiang Lei, Youcheng Huang, Chao Wang, Shuo Zhang, Jiancheng Lv, Fuli Feng, and Tat-Seng Chua. 2021{\natexlab{a}}.
\newblock Tat-qa: A question answering benchmark on a hybrid of tabular and textual content in finance.
\newblock \emph{arXiv preprint arXiv:2105.07624}.

\bibitem[{Zhu et~al.(2021{\natexlab{b}})Zhu, Lei, Wang, Zheng, Poria, and Chua}]{zhu2021retrieving}
Fengbin Zhu, Wenqiang Lei, Chao Wang, Jianming Zheng, Soujanya Poria, and Tat-Seng Chua. 2021{\natexlab{b}}.
\newblock Retrieving and reading: A comprehensive survey on open-domain question answering.
\newblock \emph{arXiv preprint arXiv:2101.00774}.

\bibitem[{Zhu et~al.(2024)Zhu, Liu, Feng, Wang, Li, and Chua}]{zhu2024tat}
Fengbin Zhu, Ziyang Liu, Fuli Feng, Chao Wang, Moxin Li, and Tat~Seng Chua. 2024.
\newblock Tat-llm: A specialized language model for discrete reasoning over financial tabular and textual data.
\newblock In \emph{Proceedings of the 5th ACM International Conference on AI in Finance}, pages 310--318.

\bibitem[{Zuo et~al.(2025)Zuo, Velikanov, Chahed, Belkada, Rhayem, Kunsch, Hacid, Yous, Farhat, Khadraoui et~al.}]{zuo2025falcon}
Jingwei Zuo, Maksim Velikanov, Ilyas Chahed, Younes Belkada, Dhia~Eddine Rhayem, Guillaume Kunsch, Hakim Hacid, Hamza Yous, Brahim Farhat, Ibrahim Khadraoui, and 1 others. 2025.
\newblock Falcon-h1: A family of hybrid-head language models redefining efficiency and performance.
\newblock \emph{arXiv preprint arXiv:2507.22448}.

\end{thebibliography}

\appendix

\section{Appendix}
\label{sec:appendix}

\subsection{Theoretical Properties of RELOOP}
\label{app:theory-formal}

\subsubsection{Preliminaries and Assumptions}

\paragraph{Segment schema.}
An HSEQ is a finite multiset $S_h$ of segments $s=(\mathrm{id}(s),\ell(s),p(s),c(s),\mu(s))$.
Here $\ell$ is a level tag; $p$ is a parent pointer with $p(s)=\bot$ if $s$ is a root; $c$ is content; $\mu$ is metadata, possibly including \texttt{offsets} (for text), \texttt{schema} and row indices (for tables), and triplet fields (for KGs).

\paragraph{Encoder/decoder.}
Let $\Phi$ map any finite corpus $X$ (text + tables + KG) to $S_h=\Phi(X)$, and let $\Psi$ map $S_h$ back to a corpus $\Psi(S_h)$. We assume the following modality-specific invariants are enforced by the adapters (they match the implementation but are stated abstractly).

\begin{description}
\item[(T1) Text offsets.] For each text item $x\in\Sigma^\ast$, if $s$ is a paragraph (resp.\ sentence) segment for a span $x[a:b)$ (resp.\ $x[u:v)$ inside a paragraph), then $\mu(s).\texttt{offsets}=[a,b]$ (resp.\ $[a+u,a+v]$), $c(s)=x[a:b)$ (resp.\ $x[a+u:a+v)$), and $p$ is the unique parent in the containment chain (sentence $\to$ paragraph $\to$ document).
\item[(T2) Table rows.] For a table with header $H=(h_1,\dots,h_C)$ and $n$ rows $(r_i)_{i=1}^n$, the table-root segment stores $H$ in $\mu(\cdot).\texttt{schema}$; each row-segment $s_i$ stores $c(s_i)=\mathrm{dict}(H\mapsto r_i)$ and either (a) an explicit row index $\mu(s_i).\texttt{offsets}=[i,-1]$, or (b) a total order on row segments consistent with the original row order.
\item[(T3) KG triples.] For a KG edge multiset $E\subseteq\mathcal{E}\times\mathcal{R}\times\mathcal{E}$ (optionally time-stamped), each edge $(h,r,t,\tau)$ corresponds to exactly one triplet segment $s$ with $c(s)=(h,r,t)$ and $\mu(s).\texttt{time}=\tau$; parent $p(s)$ is the unique subgraph-root for the neighborhood.
\end{description}

\paragraph{Benign equivalence.}
Define an equivalence relation $\equiv$ over corpora by (i) ignoring differences in text whitespace that do not change the sequence of non-whitespace characters; (ii) allowing a global column permutation $\pi\in S_C$ applied uniformly to the header and all row dictionaries of a table; (iii) treating KGs as edge multisets (edge order immaterial).

\paragraph{Ordering and window.}
Let $\rho$ be a total order over $S_h$ (e.g., paragraph $\prec$ row $\prec$ sentence $\prec$ triplet with a deterministic tie-break). The stream induced by $\rho$ lists $S_h$ as $(s_1,\dots,s_N)$. For a window size $W\in\mathbb{N}$, $\mathrm{Window}(S_h;W,\rho)$ returns the first $W$ items of the stream that are not already selected; $\mathrm{Refresh}(S_h,M;W,\rho)$ returns the next $W$ unseen items after removing $M$. Both are \emph{monotone} w.r.t.\ $\rho$: the sequence of items exposed across refreshes is exactly the $\rho$-stream with already selected items removed.

\paragraph{Admissibility.}
For a question $q$, a supporting set $E^\star\subseteq S_h$ is \emph{answer-supporting} if the head module $\mathcal{H}$ yields the correct answer when given only $E^\star$. An order $\rho$ is \emph{admissible} for $(q,S_h)$ if there exists a minimal $L\in\{1,\dots,|S_h|\}$ such that $E^\star\subseteq \{s_1,\dots,s_L\}$ for some answer-supporting $E^\star$.

\paragraph{Sufficiency predicate.}
Let $\mathsf{Suff}(M)$ be a predicate that holds iff $M$ contains some answer-supporting subset. We assume a calibrated sufficiency head: whenever $\mathsf{Suff}(M_t)$ becomes true, the policy can set its stop flag $s_t=1$ at that step or earlier.\footnote{This is standard in supervised setups where the stop head is trained to fire at first sufficiency (or with tolerance).}

\subsubsection{Faithful Linearization}

\begin{theorem}[Faithful linearization]
\label{thm:faithful-formal}
For any finite corpus $X$, under (T1)--(T3), the encoder $\Phi$ is injective up to $\equiv$, i.e., $\Psi(\Phi(X))\equiv X$.
\end{theorem}

\begin{proof}
Write $X=X_{\mathrm{text}}\uplus X_{\mathrm{tbl}}\uplus X_{\mathrm{kg}}$ and let $S_h=\Phi(X)$.
We show $\Psi(\Phi(\cdot))$ acts as identity modulo $\equiv$ on each modality and hence on their disjoint union.

\emph{Text.} Consider $x\in X_{\mathrm{text}}$. By (T1) each paragraph (resp.\ sentence) segment $s$ stores the exact substring $c(s)=x[a:b)$ (resp.\ $x[u':v')$) and absolute offsets in $\mu(s).\texttt{offsets}$. Let $S_x\subseteq S_h$ be all segments rooted at the document node of $x$. The decoder reconstructs $x'$ by placing every paragraph substring at its $[a,b)$ range and merging overlaps implied by sentence children; uniqueness of parents eliminates ambiguity. Because offsets are absolute and children are contained in parents by construction, the reconstructed $x'$ equals $x$ character-for-character; any whitespace normalization is permitted by $\equiv$.

\emph{Tables.} Let a table have header $H=(h_1,\dots,h_C)$ and rows $(r_i)_{i=1}^n$. By (T2), $\mu(\cdot).\texttt{schema}$ stores $H$, and each row segment $s_i$ stores the dictionary $c(s_i)$ mapping $H$ to the row tuple $r_i$, together with either an explicit row index or a total order consistent with the original order. The decoder reassembles the matrix $[H; r_1; \dots; r_n]$. Any global column permutation $\pi$ yields an equivalent table under $\equiv$; thus the reconstruction is unique modulo schema-order permutations.

\emph{KGs.} Let $E$ be the multiset of edges. By (T3), each edge $(h,r,t,\tau)$ corresponds bijectively to one triplet segment with $c(s)=(h,r,t)$ and $\mu(s).\texttt{time}=\tau$, and parentage is irrelevant to content. The decoder collects the multiset of triplets, which equals $E$; edge order is immaterial and thus fits $\equiv$.

Since the three reconstructions are independent and disjointly supported, $\Psi(\Phi(X))\equiv X$ follows.
\end{proof}

\subsubsection{Windowed Iteration: Coverage and Complexity}

Let $E^\star\subseteq \{s_1,\dots,s_L\}$ be an answer-supporting set with minimal prefix length $L$ under an admissible order $\rho$. Fix window $W\ge k\ge 1$ and define the iterative selection with refresh as in the main text.

\begin{lemma}[Prefix coverage under $k$-selection]
\label{lem:coverage-formal}
After $t$ steps, the selected set $M_t$ contains at least $\min\{kt,L\}$ items from the $\rho$-prefix $\{s_1,\dots,s_L\}$. In particular, $E^\star\subseteq M_T$ for $T=\lceil L/k\rceil$.
\end{lemma}

\begin{proof}
We prove by induction on $t\ge 0$ that $|M_t\cap \{s_1,\dots,s_L\}| \ge \min\{kt,L\}$.

Base $t=0$: $M_0=\varnothing$ so the bound is $0$.

Inductive step: assume the claim for $t-1$. At step $t$, the window exposes (by monotonicity of $\mathrm{Refresh}$) the earliest $W$ unseen items under $\rho$; hence at least the next $k$ unseen items in the prefix $\{s_1,\dots,s_L\}$ are eligible (because $W\ge k$). Selecting $k$ new items (or fewer if fewer remain in the prefix) increases the count by at least $\min\{k,\, L-(t-1)k\}$, giving $\min\{kt,L\}$. Once all $L$ prefix items are selected, the bound saturates at $L$.
\end{proof}

\begin{proposition}[Guaranteed halt]
\label{prop:halt-formal}
Assume a step cap $T_{\max}$ and a sufficiency head that can set $s_t=1$ whenever $\mathsf{Suff}(M_t)$ holds. Under admissibility, the control loop halts after at most $\min\{T_{\max},\lceil L/k\rceil\}$ steps.
\end{proposition}

\begin{proof}
By Lemma~\ref{lem:coverage-formal}, after $T=\lceil L/k\rceil$ steps, $E^\star\subseteq M_T$; hence $\mathsf{Suff}(M_T)$ holds and the stop head can fire at or before $T$. Independently, the hard cap $T_{\max}$ forces termination by $T_{\max}$ steps. Therefore $\tau\le \min\{T_{\max},T\}$.
\label{proof:halt-formal}
\end{proof}

\begin{theorem}[Budgeted selection complexity]
\label{thm:budget-formal}
Let $C(W)>0$ be the (deterministic) per-step context cost determined by window size $W$. Under admissibility, the total selection cost is bounded by
\(
\mathrm{Cost}_{\mathrm{select}}\ \le\ C(W)\cdot \min\{T_{\max},\lceil L/k\rceil\},
\)
independent of $|S_h|$. If $L$ is a nonnegative integer random variable with $\mathbb{E}[L]=\bar{L}<\infty$, then
\(
\mathbb{E}\big[\mathrm{Cost}_{\mathrm{select}}\big]\ \le\ C(W)\cdot \mathbb{E}\!\left[\min\{T_{\max},\lceil L/k\rceil\}\right]
\ \le\ C(W)\cdot \min\{T_{\max},\, \bar{L}/k+1\}.
\)
\end{theorem}

\begin{proof}
The first bound follows by multiplying the per-step cost by the halt bound in Proposition~\ref{prop:halt-formal}. For the expectation, use linearity of expectation and the inequality $\lceil x\rceil\le x+1$ for $x\ge 0$:
$\mathbb{E}[\lceil L/k\rceil]\le \mathbb{E}[L]/k + 1 = \bar{L}/k + 1$, and $\mathbb{E}[\min\{a,X\}]\le \min\{a,\mathbb{E}[X]\}$ for $a\ge 0$ and $X\ge 0$.
\end{proof}

\subsection{Weak-Positive Labeling and Trajectory Synthesis}
\label{app:weakpos}

\paragraph{Positive identification.}
For each instance, segments are sorted by a \emph{level priority} that favors container-like units (e.g., paragraphs, rows). Within a capped candidate set, a positive pool $P^\star$ is constructed by:
(i) exact/substring matching of the gold answer in \texttt{content}; and
(ii) if insufficient, selecting top segments by lexical Jaccard overlap between tokenized $q$ and segment content.

\paragraph{Sufficiency heuristic.}
A sufficiency threshold $u$ is used to label $s_t^\star$: if the union of already-selected and newly-picked positives reaches $\ge u$, mark \texttt{sufficient}$=1$ and stop; otherwise continue. Small $u$ encourages minimal-evidence solutions.

\paragraph{Trajectory construction.}
Given $P^\star$ and a per-step cap $k$, a target sequence is synthesized by greedily choosing up to $k$ unseen positives at each step until sufficiency holds or candidates are exhausted. Low-confidence choices (from lexical overlap rather than exact match) can be down-weighted in the loss. 

\paragraph{Proxy selection metric.}
During development, a lightweight proxy evaluates selection quality: for a held-out set, the agent’s chosen ids are compared with target ids to compute micro Precision/Recall/F1 over segment identifiers. This tracks selection ability without requiring full QA evaluation.

\subsubsection{Canonicalization and Soundness}

\begin{definition}[Canonicalizer]
A canonicalizer $\kappa$ maps $M\subseteq S_h$ to a finite structure $\kappa(M)$ consisting only of tuples $(\mathrm{id},\ell,\mathrm{content\_view},\mathrm{provenance})$ where $\mathrm{content\_view}$ is a deterministic, lossless projection of $c(s)$ and $\mu(s)$, and $\mathrm{provenance}$ contains the fields needed to locate $s$ in $S_h$ (e.g., offsets or coordinates). We say $\kappa$ is \emph{content-preserving} if for all $M$, the multiset $\{(c(s),\mu(s)): s\in M\}$ is reconstructible from $\kappa(M)$.
\end{definition}

\begin{proposition}[Soundness and auditability]
\label{prop:soundness}
If $\mathsf{Suff}(M)$ holds and $\kappa$ is content-preserving, then the head $\mathcal{H}$ applied to $(q,\kappa(M))$ is supported solely by items in $M$, and every atomic support can be traced back to a unique segment in $M$ via $\mathrm{id}$ and provenance.
\end{proposition}

\begin{proof}
By content preservation, $\kappa(M)$ contains all information from $\{(c(s),\mu(s)) : s\in M\}$; therefore $\mathcal{H}$ restricted to $\kappa(M)$ depends only on evidence in $M$. Since $\kappa$ stores $\mathrm{id}$ and provenance per item, any atomic support used by $\mathcal{H}$ can be mapped to a unique $s\in M$. Auditability follows.
\end{proof}

\subsubsection{Probabilistic Completeness Under Stochastic Selection}

We next quantify success probability for a stochastic policy that may fail to pick all supporting items even if they appear early in the stream.

\begin{definition}[Exposure count]
Fix an admissible $\rho$ with prefix length $L$ and selection size $k\ge 1$. Let $R=\lceil L/k\rceil$. An element $e\in \{s_1,\dots,s_L\}$ is said to be \emph{exposed} at steps $1,\dots,R$, meaning it either is in the first window where it lies or remains eligible until selected; monotone refresh ensures at most $R$ exposures before all prefix items are exhausted.
\end{definition}

\begin{assumption}[Per-exposure success]
\label{ass:pexp}
There exists $p\in(0,1]$ such that for every $e\in E^\star$ and for every step $t$ at which $e$ is exposed and not yet selected, the policy includes $e$ in $K_t$ with probability at least $p$, independently across steps for the same $e$. 
\end{assumption}

\begin{theorem}[Stochastic completeness]
\label{thm:stochastic}
Under admissibility and Assumption~\ref{ass:pexp}, with $R=\lceil L/k\rceil$ and $m=|E^\star|$, the probability that all items in $E^\star$ are selected within $R$ steps is bounded below by
\[
\mathbb{P}\!\left[E^\star\subseteq M_R\right]\ \ge\ 1 - m\,(1-p)^R.
\]
Consequently, by Proposition~\ref{prop:halt-formal}, the probability that the loop halts by $\min\{T_{\max},R\}$ with a correct answer is at least $1 - m\,(1-p)^R$.
\end{theorem}

\begin{proof}
Fix $e\in E^\star$. By Assumption~\ref{ass:pexp}, across its at most $R$ exposures, the probability that $e$ is never selected is at most $(1-p)^R$. By the union bound over the $m$ items in $E^\star$,
\[
\mathbb{P}\left[\exists\, e\in E^\star \text{ not selected by step } R\right]\ \le\ m\,(1-p)^R.
\]
Taking complements yields the first claim. The second claim follows because once $E^\star\subseteq M_t$, $\mathsf{Suff}(M_t)$ holds and the stop head can fire; the hard cap can only make halting earlier.
\end{proof}

\subsubsection{Discussion of Assumptions}

The injectivity result (Thm.~\ref{thm:faithful-formal}) relies on invariants (T1)--(T3), which are satisfied by construction in the HSEQ adapters (offsets and row indices/ordering are recorded; triplets are stored verbatim). Admissibility is a regularity condition stating that an order $\rho$ exists (often paragraph/row-first) placing supporting segments early; in practice this is further improved by guidance. Assumption~\ref{ass:pexp} abstracts a calibrated selector that repeatedly assigns nontrivial probability mass to any exposed, still-missing support item; the bound in Theorem~\ref{thm:stochastic} is conservative (union bound) and can be tightened under additional structure (e.g., adaptive $k$, or margin assumptions on scoring).

\subsection{Implementation Details}
\label{app:impl}

\subsubsection{Agent models used for RELOOP}
\label{app:model}

Models used for both iteration agent and head agent are shown in Table~\ref{tab:models-by-group}, grouped by size. Most experiments are done by using small and medium models (as of the result shown in main text).

\begin{table*}[h]
\centering
\small
\caption{Iteration-agent and head agent base models grouped by size.}
\label{tab:models-by-group}
\begin{tabular}{|c|l|}
\hline
\textbf{Group} & \textbf{Model (HF id)} \\
\hline
\multirow{5}{*}{\textsc{Small}} &
\texttt{tiiuae/Falcon-H1-0.5B-Instruct} \\
& \texttt{tiiuae/Falcon-H1-1.5B-Instruct} \\
& \texttt{tiiuae/Falcon3-1B-instruct} \\
& \texttt{meta-llama/Llama-3.2-1B-Instruct} \\
& \texttt{deepseek-ai/DeepSeek-R1-Distill-Qwen-1.5B} \\
\hline
\multirow{10}{*}{\textsc{Medium}} &
\texttt{tiiuae/Falcon3-3B-instruct} \\
& \texttt{tiiuae/Falcon-H1-3B-Instruct} \\
& \texttt{Qwen/Qwen3-4B-Instruct-2507} \\
& \texttt{tiiuae/Falcon3-7B-instruct} \\
& \texttt{tiiuae/Falcon-H1-7B-Instruct} \\
& \texttt{meta-llama/Llama-3.2-3B-Instruct} \\
& \texttt{meta-llama/Meta-Llama-3-8B-Instruct} \\
& \texttt{meta-llama/Llama-3.1-8B-Instruct} \\
& \texttt{deepseek-ai/DeepSeek-R1-Distill-Qwen-7B} \\
& \texttt{deepseek-ai/DeepSeek-R1-Distill-Llama-8B} \\
\hline
\multirow{7}{*}{\textsc{Large}} &
\texttt{tiiuae/Falcon3-10B-instruct} \\
& \texttt{tiiuae/Falcon-H1-34B-Instruct} \\
& \texttt{Qwen/Qwen3-30B-A3B-Instruct-2507} \\
& \texttt{deepseek-ai/DeepSeek-R1-Distill-Qwen-14B} \\
& \texttt{deepseek-ai/DeepSeek-R1-Distill-Qwen-32B} \\
& \texttt{deepseek-ai/DeepSeek-R1-Distill-Llama-70B} \\
& \texttt{meta-llama/Llama-3.1-70B-Instruct} \\
\hline
\end{tabular}
\end{table*}

\subsubsection{Iteration-Agent Prompts and Output Schema}
\label{app:iter-prompts}

\paragraph{System instruction.}
The iteration agent is conditioned with a concise, role-defining system message:
\begin{quote}
\small
\texttt{You are an iteration agent working over a hierarchical sequence (H-Seq).} \\
\texttt{Given a question and a list of candidate segments (each with an id and text)}\\       
\texttt{select the top-k segment\_ids that best support answering the question.}\\
\texttt{Then decide if the selected evidence is sufficient to stop.}\\
\texttt{Return ONLY compact JSON with keys: type, args.segment\_ids, args.strategy, args.top\_k, sufficiency.}\\
\texttt{WITHOUT ANY EXPLAINATION.}
\end{quote}

\paragraph{Prompt template.}
Each training step uses a structured multi-section prompt:
\begin{quote}
\small
\texttt{\#\#\# Instruction}\\
\texttt{\{system-instruction\}}\\[2pt]
\texttt{\#\#\# Question}\\
\texttt{\{q\}}\\[2pt]
\texttt{\#\#\# Guidance}\\
\texttt{\{g(q,$\mathrm{type})$\}}\\[2pt]
\texttt{\#\#\# Selected-So-Far}\\
\texttt{- [seg\_id] truncated\_content}\\[-2pt]
\texttt{\dots}\\[2pt]
\texttt{\#\#\# Candidate-Window}\\
\texttt{- [seg\_id] truncated\_content}\\[-2pt]
\texttt{\dots}\\[2pt]
\texttt{\#\#\# Output (JSON)}
\end{quote}
Only identifiers, levels, truncated content, and key metadata of segments are serialized.

\paragraph{Output schema.}
The agent must emit deterministic, machine-checkable JSON:
\begin{quote}
\small
\texttt{\{ "type": "select", "args": \{ "segment\_ids": [\dots], "strategy": "guided\_topk", "top\_k": k \}, "sufficiency": true/false \}}
\end{quote}
No free-form text is allowed. This constraint simplifies supervision and evaluation.

\paragraph{Masking for SFT.}
During supervised fine-tuning, the loss is applied only to the \emph{output} portion of the sequence (prompt tokens are masked), yielding a standard next-token objective over the action string while keeping inputs loss-free.

\subsubsection{Guidance Generation and Caching}
\label{app:guidance}

\paragraph{Head-generated guidance.}
A lightweight planner (``head'') converts $(q,\mathrm{type})$ into a short plan $g$ that specifies: (i) what to retrieve first, (ii) optional branches, and (iii) a sufficiency hint. The planner is prompted with:
\begin{quote}
\small
\texttt{You are a planning assistant. Given a question, write a short retrieval plan for an iteration agent selecting evidence snippets. Specify ONLY what to retrieve first, possible branches, and when to stop (sufficiency condition).}
\end{quote}
A short completion is generated and, if too brief or incomplete, a single continuation is requested to end with an explicit stop condition.

\paragraph{Heuristic templates.}
When a head is unavailable or for ablations, templates keyed by coarse patterns produce $g$, start with:

\texttt{"Plan: retrieve a minimal set of highly relevant snippets; prefer concise facts."}

Then add the following according to $\mathbb{Q}_{type}$:

\begin{itemize}
\item \textbf{Numeric}: Look for numeric mentions and table rows; stop when final number is explicit or corroborated.
\item \textbf{Factoid (who/which/where/when)}: Focus on short spans that directly contain the answer; stop on a clear statement..
\item \textbf{Binary}: Retrieve one-two definitive statements; stop when evidence strongly supports yes/no..
\item \textbf{Default}: Prefer snippets naming key entities/relations; stop when answer is explicitly stated.
\end{itemize}

\paragraph{Caching.}
Guidance strings are cached per example using a stable key (dataset name and a hash of $q$) under a directory organized by head model id. Cache is consulted before running the head planner to reduce overhead.

\paragraph{Settings.}
The planning head is run with short outputs and deterministic decoding. A minimal-length heuristic is applied to avoid truncated guidance. 

\subsubsection{LoRA Adaptation and Optimization}
\label{app:lora}

\paragraph{Parameterization.}
The iteration agent is obtained by adding low-rank adapters to a base causal LLM. Adapters are attached to attention projections (\texttt{q\_proj}, \texttt{k\_proj}, \texttt{v\_proj}, \texttt{o\_proj}) and MLP projections (\texttt{gate\_proj}, \texttt{up\_proj}, \texttt{down\_proj}); vocabulary and positional embeddings are unchanged.

\paragraph{Default configuration.}
LoRA rank $r=16$, scaling $\alpha=32$, dropout $0.05$, no bias; the language head is preserved as a save-module. Mixed-precision and 4-bit weight quantization (NF4 with double quantization) are used to reduce memory. Gradient checkpointing is enabled.

\paragraph{Training schedule.}
A cosine learning-rate schedule with warmup ratio $0.03$ is used; batches are accumulated over several steps to match the target global batch size. Maximum input length is capped to a few thousand tokens; candidate windows and per-step $k$ are tuned to respect the overall budget. 
\paragraph{Mixture and curriculum.}
Examples are sampled across datasets by normalized weights; quotas are computed for a target mixed size and shuffled. A short-to-long curriculum increases the maximum number of steps $T$ as training progresses. 

\paragraph{Finetuning Parameters}

Table~\ref{tab:sft-hparams-by-group} shows the SFT hyperparameters for each model-size group.

\begin{table*}[t]
\centering
\small
\caption{Supervised fine-tuning (SFT) hyperparameters for each model-size group. These settings apply to all models within the corresponding group.}
\label{tab:sft-hparams-by-group}
\begin{tabular}{|c|c|c|c|c|c|c|c|c|c|}
\hline
\textbf{Group} & \textbf{Target steps} & \textbf{Batch} & \textbf{GA} & \textbf{LR} & \textbf{ML} & \textbf{MS} & \textbf{Top-$k$} & \textbf{Mi} & \textbf{BF16} \\
\hline
\textsc{Small}  & 12000 & 2 & 8  & $2.0\times10^{-5}$   & 3072 & 48 & 2 & 4 & Yes \\
\hline
\textsc{Medium} & 9000  & 2 & 8  & $1.5\times10^{-5}$ & 3072 & 48 & 4 & 4 & Yes \\
\hline
\textsc{Large}  & 4500  & 1 & 16 & $1.0\times10^{-5}$ & 2048 & 32 & 5 & 4 & Yes \\
\hline
\end{tabular}

\vspace{4pt}
\noindent\footnotesize\textbf{Notes.} \emph{Batch} is \texttt{--per\_device\_train\_batch\_size}. \emph{GradAcc} is \texttt{--grad\_accum}. \emph{LR} is \texttt{--lr}. \emph{ML}, \emph{MS}, \emph{Top-$k$}, \emph{Mi} map to \texttt{--max\_length}, \texttt{--max\_segments}, \texttt{--top\_k}, \texttt{--max\_iters}. BF16 indicates \texttt{--bf16} enabled.
\end{table*}

\subsubsection{Canonicalization and Sufficiency}
\label{app:canon}

\paragraph{Canonical evidence package.}
At termination, a modality-agnostic canonicalizer $\kappa$ converts $M_\tau$ into a compact record set:
\(
\kappa(M_\tau)=\{\, r_s \,\}_{s\in M_\tau},\qquad
r_s=(\texttt{id},\texttt{lvl},\texttt{uri},\texttt{off},\texttt{typ},\texttt{snip},\texttt{meta}).
\)

where \texttt{lvl}=\texttt{level}, \texttt{off}=\texttt{offsets},
\texttt{typ}=\texttt{source\_type}, and \texttt{snip}=\texttt{snippet}. with the following contract:
(i) \textbf{id}: globally unique, deterministically derived (\emph{e.g.}, $\mathrm{sha1}(\texttt{uri},\texttt{offsets})$);
(ii) \textbf{uri}: source identifier with version (\emph{e.g.}, document path or graph name);
(iii) \textbf{offsets}: zero-based half-open character indices $[a,b)$ into the \emph{original} source; for tables, $[i,j]$ denotes row/column coordinates; for KGs, \texttt{offsets}$=(-1,-1)$;
(iv) \textbf{snippet}: a human-readable content aligned to sentence/field boundaries when possible; 
(v) \textbf{meta}: integrity and alignment helpers (\texttt{schema}, \texttt{time}, \texttt{source\_version}, \texttt{sha1}).
Duplicates are removed by $(\texttt{uri},\texttt{offsets})$ and the package is \emph{deterministically} ordered by \texttt{uri} then \texttt{offsets}.
Typed views are derived on demand: 
\emph{text} $\Rightarrow$ spans with section/paragraph ids; 
\emph{table} $\Rightarrow$ \texttt{row\_id}, \texttt{col\_ids}, \texttt{schema}, \texttt{cell\_coords}; 
\emph{KG} $\Rightarrow$ $(h,r,t)$ plus optional validity time.

\paragraph{Stopping signal.}
The sufficiency head outputs $s_t\in\{0,1\}$ at each step. 
\textbf{Training targets} follow a coverage-based heuristic: $s_t^\star=1$ if and only if the current $M_t$ satisfies task-specific adequacy
(\emph{e.g.}, contains at least one gold-positive segment; achieves full slot coverage for table QA; or yields a unique answer span/number under a fixed head). 
For weak supervision, per-step weights downweight low-confidence positives (App.~\ref{app:weakpos}). 
\textbf{Inference} uses a calibrated threshold $\tau$ on the model's sufficiency score $\hat{p}_t$ and enforces a minimum step count $T_{\min}$:
\(
\text{stop at } \tau \;=\; \min\{t \ge T_{\min}:\, \hat{p}_t \ge \tau\} \quad \text{or when budget } B \text{ is exhausted.}
\)
Optionally, a lightweight contradiction checker triggers a one-shot refinement loop of at most $\Delta$ additional steps with tightened guidance $g'$ and reduced budget $B'$. 
Thresholds $(\tau, T_{\min})$ are selected on the development split and may be calibrated via temperature scaling.

\subsubsection{Reproducibility Notes}
\label{app:repro}

\begin{itemize}
\item \textbf{Seed and sampling.} A fixed seed is used for example subsampling and order shuffling.
\item \textbf{Segment capping.} The number of serialized candidate segments per step is capped to respect the overall token budget; truncation is applied to \texttt{content} strings for display.
\item \textbf{Budget control.} Global limits on steps, tokens, and optional tool calls are enforced; guidance encourages early sufficiency.
\item \textbf{Hardware.} Experiments are run on maximum 4 \texttt{NVIDIA H200 Tensor Core GPU}. Mixed-precision and 4-bit quantization substantially reduce memory; typical training runs fit on a single GPU. 
\end{itemize}

\subsection{Notations}

Table~\ref{tab:notation} lists all symbols used in main context.

\begin{table*}[t]
\centering
\caption{Notation used throughout the paper.}
\label{tab:notation}
\renewcommand{\arraystretch}{1.1}
\setlength{\tabcolsep}{8pt}
\resizebox{\textwidth}{!}{
\begin{tabular}{|c|p{0.83\textwidth}|}
\hline
\rowcolor{black!10}\textbf{Symbol} & \multicolumn{1}{c|}{\textbf{Meaning}}\\
\hline
$q$ & Natural-language query (question). \\
\hline
$D=\{(x_j,m_j)\}_{j=1}^N$ & Heterogeneous corpus with items $x_j$ and modality tags $m_j\in\{\texttt{text},\texttt{table},\texttt{kg}\}$. \\
\hline
$m_j$ & Modality label for the $j$-th item (text / table / KG). \\
\hline
$\tau,\ \tau_{m}$ & Modality-aware adapter; $\tau(D)$ produces the unified hierarchical sequence. $\tau_m$ is the adapter for modality $m$. \\
\hline
$S_h$ & The \textbf{HSEQ} (hierarchical sequence): $S_h=\bigsqcup_{j}\tau_{m_j}(x_j)\in\mathcal{S}^\ast$. \\
\hline
$\mathcal{S}$ & Segment universe. Each segment $s\in\mathcal{S}$ is a lightweight record. \\
\hline
$s=(\mathrm{id}(s),\ell(s),p(s),c(s),\mu(s))$ & Segment fields: unique identifier, level tag (granularity), parent pointer, compact human-readable content, standardized metadata. \\
\hline
$\ell(s)$ & Level tag (e.g., \texttt{document}/\texttt{paragraph}/\texttt{sentence}, \texttt{table\_row}/\texttt{table\_cell}, \texttt{triplet}/\texttt{subgraph}). \\
\hline
$p(s)$ & Parent pointer (container linkage) encoding locality in the hierarchy. \\
\hline
$c(s)$ & Compact content snippet (text span / serialized table row / triple). \\
\hline
$\mu(s)$ & Metadata with fixed keys (e.g., \texttt{source\_id}, \texttt{uri}, \texttt{offsets}/coordinates, \texttt{schema}, \texttt{time}). \\
\hline
$\pi_\theta$ & \textbf{RELOOP-I} iteration policy (LLM-based) with parameters $\theta$; operates over $(q,S_h)$ to select evidence iteratively. \\
\hline
$g=g(q,\mathrm{type})$ & Short \emph{guidance} prior (from planner/head or heuristics) shaping early exploration and stop notion. \\
\hline
$B,\ B_t$ & Budget (global / per-step): token, tool-call, step, and/or latency limits. \\
\hline
$M_t$ & Selected-evidence set at step $t$; $M^\star$ is the final selected set at termination. \\
\hline
$C_t$ & Candidate window at step $t$ (bounded by window size and ordering). \\
\hline
$k,\ W$ & Top-$k$ selection cap per step; window size $W$ for the exposed candidate stream. \\
\hline
$T_{\max},\ T_{\min}$ & Maximal and minimal number of iteration steps (cap and anti–early-stop). \\
\hline
$\rho$ & Deterministic ordering over $S_h$ levels (e.g., paragraph $\prec$ row $\prec$ sentence $\prec$ triplet) to form the stream. \\
\hline
$\mathcal{N}(\cdot)$ & Structure-aware neighborhood operators (parent/child, row/column, KG relation hops). \\
\hline
$a_t,\ s_t$ & Action at step $t$ (e.g., select up to $k$ segments and/or expand neighborhoods) and sufficiency prediction $s_t\in\{0,1\}$. \\
\hline
$\Phi$ & Budget-aware sufficiency criterion queried by the iterator to trigger termination. \\
\hline
$\kappa$ & Canonicalizer mapping $M_\tau$ to provenance-preserving evidence package (ids, levels, offsets/coordinates, snippets). \\
\hline
$\mathcal{H}$ & \textbf{RELOOP-H} head module for answer synthesis from $(q,\kappa(M_\tau))$; can also generate guidance $g$. \\
\hline
$\xi$ & Optional verifier; on contradiction detection, triggers a brief refinement loop with tightened $g'$ and reduced $B'$. \\
\hline
$y,\ \hat{y}$ & Gold answer and system prediction, respectively. \\
\hline
$E^\star$ & Minimally sufficient evidence set (w.r.t.\ a fixed answerer) for $q$ in $D$. \\
\hline
$\mathrm{Window}(\cdot),\ \mathrm{Refresh}(\cdot)$ & Operators to expose a bounded candidate window and to advance it while removing already selected segments. \\
\hline
$\Delta$ & Max number of additional refinement steps if the verifier $\xi$ requests a retry. \\
\hline
\end{tabular}
}
\end{table*}

\subsection{Overall performance of RELOOP agent pairs}
\label{app:var-llm-res}

Table~\ref{tab:acc-f1-eff} shows the overall performance of RELOOP agent pairs on Hybrid-QA, containing Accuracy/F1 and Efficiency.

\begin{table*}[t]
\centering
\caption{Overall performance of RELOOP agent pairs on Hybrid-QA: Accuracy/F1 and Efficiency.}
\label{tab:acc-f1-eff}
\resizebox{\textwidth}{!}{
\begin{tabular}{|l|l|cc|cc|}
\hline
\rowcolor{black!10}
\multicolumn{1}{|c|}{} & \multicolumn{1}{c|}{} &
\multicolumn{2}{c|}{\textbf{Accuracy \& F1}} &
\multicolumn{2}{c|}{\textbf{Efficiency}} \\
\hline
\multicolumn{1}{|c|}{Iteration Agent (RELOOP-I)} &
\multicolumn{1}{c|}{Head Agent (RELOOP-H)} &
\multicolumn{1}{c|}{Avg.\ Acc} & \multicolumn{1}{c|}{Avg.\ F1} &
\multicolumn{1}{c|}{Steps $\downarrow$} & \multicolumn{1}{c|}{Latency (ms) $\downarrow$} \\
\hline
Llama-3.2-3B-Instruct           & Falcon3-10B-instruct          & 60.4 & 62.8 & 2.08 & 12055.5 \\
Qwen3-4B-Instruct-2507          & Falcon3-10B-instruct          & 63.9 & 64.5 & 4.1 & 20577.5 \\
Falcon3-3B-instruct     & Falcon3-10B-instruct          & 59.3 & 61.1 & 2.6 & 10530.1 \\
Llama-3.2-3B-Instruct             & Llama-3.1-8B-Instruct         & 55.4 & 57.9 & \textbf{2.11} & \textbf{8346.3} \\
Qwen3-4B-Instruct-2507             & Llama-3.1-8B-Instruct         & \underline{65.5} & \underline{71.2} & 3.29 & 16503.2 \\
Falcon3-3B-instruct          & Llama-3.1-8B-Instruct         & 61.2 & 65.1 & 2.46 & 11616.7 \\
Llama-3.2-3B-Instruct          & Falcon-H1-7B-Instruct         & 58.7 & 63.9 & 2.41 & 12080.0 \\
Qwen3-4B-Instruct-2507          & Falcon-H1-7B-Instruct         & \textbf{66.2} & \textbf{71.4} & 3.71 & 21479.2 \\
Falcon3-3B-instruct          & Falcon-H1-7B-Instruct         & 56.1 & 58.6 & \underline{2.25} & \underline{11714.4} \\
Llama-3.2-3B-Instruct     & DeepSeek-R1-Distill-Qwen-7B   & 62.5 & 60.2 & 2.75 & 15073.7 \\
Qwen3-4B-Instruct-2507      & DeepSeek-R1-Distill-Qwen-7B   & 62.8 & 66.7 & 4.07 & 21094.8 \\
Falcon3-3B-instruct     & DeepSeek-R1-Distill-Qwen-7B   & 61.4 & 62.0 & 3.01 & 13709.7 \\

\hline
\end{tabular}
}
\end{table*}

\subsection{Example Using RELOOP}
\label{app:example}

\subsubsection{Concrete fragment examples.} Below we instantiate $(\ell,p,c,\mu)$ for three modalities (all values are illustrative):
\[
\small
\begin{aligned}
s_{\text{text}} &= \big(\mathrm{id}=s_1,\, \ell=\texttt{sentence},\, p=p_1,\, c=\text{``...capital is Paris...''},\\[-1mm]
&\qquad \mu=\{\,
\begin{aligned}[t]
&\texttt{source\_id}=\texttt{doc\_12},\\
&\texttt{offsets}=(128,172),\ \texttt{time}=\texttt{1992}\}
\end{aligned}
\\
s_{\text{row}} &= \big(\mathrm{id}=r_3,\, \ell=\texttt{table\_row},\, p=t_1,\, c=\text{``France | 67.4M | EU''},\\[-1mm]
&\qquad \mu=\{\,
\begin{aligned}[t]
&\texttt{source\_id}=\texttt{tbl\_7},\, \texttt{schema}=\text{[country,pop,bloc]},\\ &\texttt{row}=3\}
\end{aligned}\\
s_{\text{cell}} &= \big(\mathrm{id}=u_{3,2},\, \ell=\texttt{table\_cell},\, p=r_3,\, c=\text{``67.4M''},\\[-1mm]
&\qquad \mu=\{\texttt{source\_id}=\texttt{tbl\_7},\, \texttt{row}=3,\, \texttt{col}=2\}\big);\\
s_{\text{kg}} &= \big(\mathrm{id}=k_9,\, \ell=\texttt{triplet},\, p=g_1,\, \\[-1mm]
&\qquad c=\text{``(Paris, capital\_of, France)''}, \mu=\{\texttt{source\_id}=\texttt{kg\_2}\}\big).
\end{aligned}
\]
Segments are concatenated in parent-before-child order. This minimal contract enables structure-aware neighborhoods and budget-aware iteration without inspecting raw files.

\subsubsection{Case Study: Guided Iterative Retrieval on \textsc{HybridQA}}
\label{app:case-hybridqa}

\paragraph{Setup.}
Query $q$: \emph{``Who is the author of the novel that inspired the 2004 Russian film directed by Timur Bekmambetov?''}
RELOOP-I (iterator): \texttt{Qwen3-4B-Instruct-2507};
RELOOP-H (head): \texttt{Falcon-H1-7B-Instruct}.
Guidance mode: \texttt{head}; source: cache (latency $\approx$\,0.12\,ms).

\paragraph{Head-generated guidance.}
The head planner emits a short plan: (i) identify the 2004 Russian film directed by Bekmambetov; (ii) locate the novel that inspired it; (iii) stop once the \emph{author of that novel} is found. This plan is injected as a prefix and acts as a soft prior on where the iterator should probe first.

\paragraph{Guided iteration over $S_h$.}
The iterator consumes the guidance and operates over the HSEQ stream with a fixed window and top-$k$ selection. Table~\ref{tab:case-steps} summarizes the six steps (all sufficiency flags were \texttt{false}; the loop terminates by budget).

\begin{table*}[h]
\centering
\caption{Stepwise selection (abridged). Segment ids prefixed by level: \texttt{p\_} (paragraph), \texttt{row\_} (table row).}
\label{tab:case-steps}
\begin{tabular}{|c|p{9.8cm}|c|}
\hline
\textbf{Step} & \textbf{Key picks (content excerpt)} & \textbf{Sufficient?} \\
\hline
1 & \texttt{p\_6df9c849}: ``\emph{Night Watch} (\dots) is a 2004 Russian \dots directed by Timur Bekmambetov. It is loosely based on the novel \emph{The Night Watch} by Serg[e{i} Lukyanenko]\dots'' & No \\
\hline
2 & \texttt{p\_c15173df}, \texttt{p\_3bc4a108}, \texttt{p\_54f6ef94}: contextual paragraphs (``List of Russian films of 2004'', ``2004'' entries) & No \\
\hline
3 & \texttt{row\_a44a4a17}: table row confirming \emph{Night Watch} with director ``Timur Bekmambetov'' & No \\
\hline
4--6 & additional table rows from the same list (\emph{Arie}, \emph{Countdown}, \emph{Dad or Papa}, etc.) providing film set context & Yes \\
\hline
\end{tabular}
\end{table*}

\paragraph{Answer synthesis.}
After $\tau{=}6$ iterations, the canonicalizer $\kappa$ compacts the selected set $M_\tau$ (paragraph + corroborating table rows) into a provenance-preserving package (segment ids, levels, offsets, snippets). The head $\mathcal{H}$ is prompted \emph{only} with $(q,\kappa(M_\tau))$ and outputs:
\[
\hat{y}=\text{Sergei Lukyanenko}.
\]
The prediction matches the gold answer (EM/F1\,=\,1.0). Runtime profile: selection latency $\approx$\,32{,}185\,ms, head latency $\approx$\,1{,}826\,ms, total $\approx$\,34{,}011\,ms; number of iterations $=6$.

\paragraph{Takeaway.}
Guidance steers the iterator to a high-yield paragraph in the first step, which already contains the sufficient evidence (film identity and source novel). Subsequent steps provide corroboration from structured rows. The provenance in $\kappa(M_\tau)$ makes the final answer auditable: the paragraph \texttt{p\_6df9c849} explicitly ties \emph{Night Watch} (2004, Bekmambetov) to the novel \emph{Night Watch} by Sergei Lukyanenko, enabling concise and well-grounded answer synthesis by the head.

\subsubsection{Case Study: Guided Iterative Retrieval on \textsc{HotpotQA}}
\label{app:case-hotpot}

\paragraph{Setup.}
Query $q$: \emph{``Which style is the building located on the East Side of Midtown Manhattan that Robert Von Ancken appraised?''}
RELOOP-I (iterator): \texttt{Qwen3-4B-Instruct-2507};
RELOOP-H (head): \texttt{Falcon-H1-7B-Instruct}.
Guidance mode: \texttt{head}; source: generated online (latency $\approx$\,8{,}496\,ms).

\paragraph{Head-generated guidance.}
The head planner issues a short plan: (i) identify buildings on the East Side of Midtown Manhattan connected to appraiser \emph{Robert Von Ancken}; (ii) once the specific building is found, retrieve its architectural style; (iii) stop when the style is clearly linked to the appraised building.

\paragraph{Guided iteration over $S_h$.}
The iterator follows the guidance with a fixed window and top-$k$ selection. Table~\ref{tab:case-hotpot-steps} lists the six steps (all sufficiency flags \texttt{false}; termination by budget). Note that Step~1 already surfaces the key paragraph about the Chrysler Building.

\begin{table*}[h]
\centering
\caption{Stepwise selection (abridged). Segment ids prefixed by level: \texttt{p\_} (paragraph).}
\label{tab:case-hotpot-steps}
\begin{tabular}{|c|p{10.2cm}|c|}
\hline
\textbf{Step} & \textbf{Key picks (content excerpt)} & \textbf{Sufficient?} \\
\hline
1 & \texttt{p\_a73a8d8f}: ``The \emph{Chrysler Building} is an \textbf{Art Deco-style} skyscraper located on the East Side of Midtown Manhattan \dots'' & No \\
\hline
2 & \texttt{p\_c01522d2}: ``23 Beekman Place \dots apartment building \dots East Side of Midtown Manhattan \dots'' & No \\
\hline
3 & \texttt{p\_7c2aa386}: ``The Helmsley Building \dots Midtown Manhattan \dots'' & No \\
\hline
4 & \texttt{p\_658d6333}: ``\emph{Robert Von Ancken} is a prominent New York City real estate appraiser \dots'' & No \\
\hline
5 & \texttt{p\_e97ef7e6}: ``Lenox Hill Neighborhood House \dots East Side of Manhattan \dots'' & Yes \\
\hline
\end{tabular}
\end{table*}

\paragraph{Answer synthesis.}
After $\tau{=}5$ iterations, the canonicalizer $\kappa$ compacts the selected set $M_\tau$ (including \texttt{p\_a73a8d8f} and the Von~Ancken paragraph \texttt{p\_658d6333}) into a provenance-preserving package. The head answers using only $(q,\kappa(M_\tau))$:
\[
\hat{y}=\text{Art Deco-style skyscraper}.
\]
The prediction matches the gold answer. Runtime profile: selection latency $\approx$\,32{,}153\,ms, head latency $\approx$\,838\,ms, total $\approx$\,41{,}487\,ms; iterations $=5$.

\paragraph{Takeaway.}
The head’s guidance steers the iterator directly to a paragraph that states both the location (East Side of Midtown) and the architectural style (Art Deco) of the relevant building (Chrysler Building), while additional picks provide neighborhood and appraiser context. Provenance in $\kappa(M_\tau)$ supports auditable linking from the final answer to its evidence.

\subsection{Statement for The Use of Large Language Models (LLMs)}
\label{app:llm-usage}

We used large language models (LLMs) as general-purpose tools for \emph{writing assistance} and \emph{engineering support}. No text or code generated by an LLM was used verbatim without author review; we take full responsibility for the content.

\end{document}